\documentclass[11pt,letterpaper]{article}

\usepackage{thmtools}
\usepackage{thm-restate}
\usepackage{bm}
\usepackage{bbm}
\usepackage{mathrsfs}           
\usepackage{vmargin,fancyhdr}   
\usepackage{tikz}
\usetikzlibrary{positioning,chains,fit,shapes,calc}
\usepackage{algorithm}
\usepackage{algpseudocode}

\usepackage{subcaption}
\usepackage{enumerate}

\usepackage{amsmath,amssymb, amsthm}    
\usepackage{verbatim}           
\usepackage{xspace}             
\usepackage{graphicx,float}     
\usepackage{ifthen,calc}        
\usepackage{textcomp}           
\usepackage{fancybox}           
\usepackage{hhline}             
\usepackage{float}              

\usepackage[vflt]{floatflt}     
\usepackage[small,compact]{titlesec}
\usepackage{setspace}
\usepackage{mathrsfs}

\usepackage{multirow}
\usepackage{footnote}
\makesavenoteenv{tabular}
\makesavenoteenv{table}

\usepackage{color}
\definecolor{ForestGreen}{rgb}{0.1333,0.5451,0.1333}
\usepackage{thumbpdf}
\usepackage[letterpaper,
colorlinks,linkcolor=ForestGreen,citecolor=ForestGreen,
backref,
bookmarks,bookmarksopen,bookmarksnumbered]
{hyperref}

\setpapersize{USletter}
\setmarginsrb{1in}{.5in}        
{1in}{.5in}        
{.25in}{.25in}     
{.25in}{.5in}      
\newcommand{\showccc}[0]{0}
\newcommand{\ccc}[2][nothing]{
	\ifthenelse{\showccc=0}{}{
		\ensuremath{^{\Lsh\Rsh}}\marginpar{\raggedright\tiny\textsf{%
				\ifthenelse{\equal{#1}{nothing}}{}{\textbf{#1}\\}#2}}}}
\newcounter{hours}\newcounter{minutes}
\newcommand{\hhmm}{%
	\setcounter{hours}{\time/60}%
	\setcounter{minutes}{\time-\value{hours}*60}%
	\ifthenelse{\value{hours}<10}{0}{}\thehours:%
	\ifthenelse{\value{minutes}<10}{0}{}\theminutes}
\ifthenelse{\showccc=0}{\rhead{}}{\rhead{\today \ [\hhmm]}}
\newtheorem{theorem}{Theorem}[section]

\newtheorem{proposition}[theorem]{Proposition}
\newtheorem{corollary}[theorem]{Corollary}
\newtheorem{definition}[theorem]{Definition}

\newtheorem{lemma}[theorem]{Lemma}
\newtheorem{fact}[theorem]{Fact}

\newtheorem{question}[theorem]{Question}

\newtheorem{assumption}[theorem]{Assumption}

\usepackage{float}
\floatstyle{plain}\newfloat{myfig}{t}{figs}[section]
\floatname{myfig}{\textsc{Figure}}
\floatstyle{plain}\newfloat{myalg}{H}{algs}[section]
\floatname{myalg}{}
\newcommand{\defeq}{\stackrel{\mathrm{{\scriptscriptstyle def}}}{=}}
\newcommand{\norm}[1]{\left\lVert#1\right\rVert}
\newcommand{\inprod}[2]{\left\langle#1, #2\right\rangle}

\newcommand{\R}{\mathbb{R}}

\newcommand{\diag}[1]{\textbf{\textup{diag}}(#1)}
\newcommand{\half}{\frac{1}{2}}

\newcommand{\E}{\mathbb{E}}
\newcommand{\Var}{\textup{Var}}

\newcommand{\Nor}{\mathcal{N}}

\newcommand{\tOh}[1]{\tilde{O}\left(#1\right)}

\newcommand{\tvd}[2]{\norm{#1 - #2}_{\textup{TV}}}
\newcommand{\pis}{\pi^*}

\newcommand{\epis}{\E_{\pis}}

\newcommand{\lf}{\texttt{Leapfrog}}
\newcommand{\tx}{\tilde{x}}
\newcommand{\tv}{\tilde{v}}
\newcommand{\ham}{\mathcal{H}}

\newcommand{\tran}{\mathcal{T}}
\newcommand{\prop}{\mathcal{P}}

\newcommand{\eps}{\epsilon}

\newcommand{\tmix}{T_{\textup{mix}}}
\newcommand{\tpi}{\tilde{\pi}}

\definecolor{burntorange}{rgb}{0.8, 0.33, 0.0}

  \usepackage{nth}
  \usepackage{intcalc}

\usepackage{url}

\begin{document}

	\begin{titlepage}
		\def\thepage{}
		\thispagestyle{empty}
		
		\title{Logsmooth Gradient Concentration and Tighter Runtimes for Metropolized Hamiltonian Monte Carlo} 
		
		\date{}
		\author{
			Yin Tat Lee \\
			University of Washington \\
			{\tt yintat@uw.edu} 
			\and
			Ruoqi Shen \\
			University of Washington \\
			{\tt shenr3@cs.washington.edu} 
			\and
			Kevin Tian\thanks{Part of this work was done when KT was visiting the University of Washington.} \\
			Stanford University \\
			{\tt kjtian@stanford.edu}
		}
		
		\maketitle
		
		\abstract{
We show that the gradient norm $\norm{\nabla f(x)}$ for $x \sim \exp(-f(x))$, where $f$ is strongly convex and smooth, concentrates tightly around its mean. This removes a barrier in the prior state-of-the-art analysis for the well-studied Metropolized Hamiltonian Monte Carlo (HMC) algorithm for sampling from a strongly logconcave distribution \cite{DwivediCWY18}. We correspondingly demonstrate  that Metropolized HMC mixes in $\tOh{\kappa d}$ iterations\footnote{We use $\tilde{O}$ to hide logarithmic factors in problem parameters.}, improving upon the $\tilde{O}(\kappa^{1.5}\sqrt{d} + \kappa d)$ runtime of \cite{DwivediCWY18, ChenDWY19} by a factor $(\kappa/d)^{1/2}$ when the condition number $\kappa$ is large. Our mixing time analysis introduces several techniques which to our knowledge have not appeared in the literature and may be of independent interest, including restrictions to a nonconvex set with good conductance behavior, and a new reduction technique for boosting a constant-accuracy total variation guarantee under weak warmness assumptions. This is the first high-accuracy mixing time result for logconcave distributions using only first-order function information which achieves linear dependence on $\kappa$; we also give evidence that this dependence is likely to be necessary for standard Metropolized first-order methods.
}

	\end{titlepage}

	\section{Introduction}
\label{sec:intro}

Sampling from a high-dimensional logconcave distribution is a fundamental computational task, central to a variety of applications in disciplines such as statistics, machine learning, operations research, and theoretical computer science \cite{AndrieuFDJ03, BrooksGJM11}. The important problem of designing efficient samplers has received significant recent attention, due to the widespread use of Bayesian methods in modern large-scale machine learning algorithms \cite{Barber12}, and the fact that directly computing posterior densities in these methods is often intractable. Correspondingly, recent research efforts have focused on giving convergence guarantees for Markov Chain Monte Carlo (MCMC) methods, for sampling from a variety of desirable structured distributions arising in both theory and practice.

The specific problem we address in this paper is determining the mixing time of the well-studied Metropolized Hamiltonian Monte Carlo (HMC) algorithm\footnote{Metropolized HMC also refers to a family of algorithms which takes multiple \emph{leapfrog} steps, see Algorithm~\ref{alg:mhmc}. In this work, we study the variant which takes one leapfrog step, to analyze convergence behavior under minimal assumptions on the log-density (i.e.\ in absence of higher-derivative bounds past smoothness).}, when sampling from a target distribution whose log-density is smooth and strongly concave. Indeed, as it is the default sampler implementation in a variety of popular packages \cite{Abadi16, CarpenterGHLG17}, understanding Metropolized HMC is of high practical import. Moreover, the specific setting we study, where the target distribution has a density proportional to $\exp(-f)$ for function $f$ with quadratic upper and lower bounds, is commonplace in applications arising from multivariate Gaussians, logistic regression models, and structured mixture models \cite{DwivediCWY18}. This setting is also of great theoretical interest because of its connection to a well-understood setting in convex optimization \cite{Nesterov03}, where matching upper and lower bounds have long-been known. Similar guarantees are much less well-understood in sampling settings, and exploring the connection is an active research area (e.g. \cite{MaCJFJ18, Talwar19} and references therein). Throughout the introduction, we will refer to this setting as the ``condition number regime'' for logconcave sampling, as without a finite condition number, black-box sampling guarantees exist, but typically have a large dimension dependence in the mixing time \cite{Vempala05}.

\subsection{Previous work}

Many algorithms have been proposed for sampling from logconcave distributions, mainly falling into two categories: zeroth-order methods and first-order methods. Zeroth-order methods only use the information on the density of the distribution by querying the value of $f$ to inform the algorithm trajectory. Well-studied zeroth-order methods include the ball walk \cite{lovasz1993random, lovasz1990mixing}, hit-and-run \cite{belisle1993hit, lovasz1999hit, lovasz2006hit} and Metropolized random walk \cite{mengersen1996rates, roberts1996geometric}. While possibly preferable in some cases due to weaker assumptions on both the class of target distributions and oracle access to the density, these methods typically have a large polynomial dependence on the dimension, and do not exploit the additional benefits afforded by having a finite condition number. For distributions with additional structure, sampling algorithms may wish to exploit said structure.

First-order methods have access to the gradient information of $f$ in addition to the value of $f$ at a query point. This class of methods usually involves simulating a continuous Markov process whose stationary distribution is exactly the target distribution. The Langevin dynamics (LD) \cite{gelfand1991recursive}, its underdamped variant (ULD), and Hamiltonian Monte Carlo (HMC) \cite{kramers1940brownian, neal2011mcmc} are among the most well-studied continuous-time Markov processes which converge to a specified logconcave measure, and sampling algorithms based on first-order information typically model their trajectories off one of these processes. To simulate a random process in discrete time, one approach is to choose a small-enough step size so that the behavior of the discrete Markov process is not too different from that of the original Markov process over a small time interval. This discretization strategy is typical of sampling algorithms with a polynomial dependence on $\eps^{-1}$, where $\eps$ is the target total variation distance to the stationary distribution \cite{dalalyan2017theoretical, cheng2017underdamped, durmus2019high, mou2019high,mangoubi2017rapid, lee2018algorithmic, chen2019optimal, shen2019randomized}. However, for precise values of $\eps$, bounding the error incurred by the discretization is typically not enough, leading to prohibitively large runtimes.

On top of the discretization, one further can apply a Metropolis-Hastings filter to adjust the stationary distribution of the Markov process, so that the target distribution is attained in the long run. Studying the non-asymptotic behavior of Metropolized variants of the Langevin dynamics and HMC has been considered in a large number of recent works \cite{roberts1996exponential, roberts1996geometric, pillai2012optimal, bou2013nonasymptotic, xifara2014langevin, DwivediCWY18, ChenDWY19}. Indeed, the standard discretizations of these methods are identical, which was observed in prior work (see Appendix~\ref{app:equivalence}); we will refer to them both as Metropolized HMC. The works which inspired this study in particular were due to \cite{DwivediCWY18, ChenDWY19}, which showed that the mixing time of Metropolized HMC was bounded by roughly $\max(\kappa^{1.5}\sqrt{d}, \kappa d)$, with logarithmic dependence on the target accuracy $\eps$, where $\kappa$ is the \emph{condition number}\footnote{The condition number of a function is the ratio of its smoothness and strong convexity parameters, and is the standard parameter in measuring the complexity of algorithms in sampling and optimization in this regime.} of the negative log-density $f$. In the $\textrm{poly}(\eps^{-1})$ runtime regime, the recent work \cite{DurmusMM19} obtains a total variation mixing time bound which scales as $\tilde{O}(\kappa d^2/\eps^4)$, which is to our knowledge the only bound known with linear dependence on $\kappa$; on the other hand, \cite{shen2019randomized} gives an algorithm that depends on $\kappa^{7/6}$ for Wasserstein-$2$ distance, but with better dependence on all other parameters (see Table~\ref{tb: mixing_time}).

By a plausible assumption on the existence of a gap between the complexity of sampling and optimization in the logconcave setting, it is reasonable to believe that a linear dependence on $\kappa$ is necessary. More specifically, it is well-known that gradient-based optimization algorithms require at least $\min(d, \sqrt{\kappa})$ queries to an oracle providing first-order information \cite{Bubeck15}; for the worst-case instance, a quadratic in the graph Laplacian of a length-$d$ path, there is a corresponding quadratic gap with sampling a uniform point via a random walk, which mixes in roughly $d^2$ iterations. We believe understanding the tight dependence of the mixing time of popular sampling algorithms on natural parameters such as the condition number is fundamental to the development of the field of sampling, just as characterizing the tight complexity of convex optimization algorithms has resulted in rapid recent progress in the area, by giving researchers goalposts in algorithm design. To that end, this work addresses the following question.

\begin{question}
	\label{q:hmc}
	What is the mixing time of the Metropolized HMC algorithm?
\end{question}

We give a comparison of (selected recent) prior work in Table~\ref{tb: mixing_time}; for a more complete discussion, we refer the reader to the excellent discussion in \cite{DwivediCWY18, ChenDWY19}. We note that for the last two rows, the dependence on $\eps$ is logarithmic, and the notion of mixing is in total variation distance, a much stronger notion than the Wasserstein metric used in all other runtimes listed. We omit logarithmic factors for simplicity. We remark that several works obtain different rates under stronger assumptions on the log-density $f$, such as higher-order smoothness (e.g. a Lipschitz Hessian) or moment bounds; as this work studies the basic condition number setting with no additional assumptions, we omit comparison to runtimes of this type.

\begin{table}[h]
	
	\centering{}%
	\begin{tabular}{ |p{9cm}|c|c| } 
		\hline
		\centering{\textbf{Algorithm} }& \textbf{Mixing Time}  & \textbf{Metric} \\ 
		\hline
		Langevin Diffusion \cite{dalalyan2017theoretical} & $\kappa^2/\eps^2$ & \multirow{7}{*}{$W_2$\footnote{With the exception of \cite{DurmusMM19}, these results are measured in the 2-Wasserstein distance.}} \\
		\cline{1-2} 
		High-Order Langevin Diffusion \cite{mou2019high} &  $\kappa^{19/4}/\eps^{1/2}+\kappa^{13/3}/\eps^{2/3}$ & \\
		\cline{1-2} 
		HMC 1 (Collocation Method)  \cite{lee2018algorithmic}& $\kappa^{1.5}/\epsilon$ &\\
		\cline{1-2} 
		HMC 2 (Collocation Method)  \cite{chen2019optimal}& $\kappa^{1.5}/\epsilon$& \\
		\cline{1-2}  
		ULD 1 (Euler Method)  \cite{cheng2017underdamped} & $\kappa^{1.5}/\eps$ &\\
		\cline{1-2} 
		ULD 2 (Euler Method)  \cite{dalalyan2018sampling} & $\kappa^{1.5}/\eps + \kappa^2$ &\\
		\cline{1-2} 
		ULD 3 (Random Midpoint Method) \cite{shen2019randomized} & $\kappa^{7/6}/\epsilon^{1/3}+\kappa/\epsilon^{2/3}$& \\
		\hline  
		Unadjusted Langevin Dynamics \cite{DurmusMM19} & $ \kappa d^2 /\eps^4$ & \multirow{3}{*}{TV}\\ 
		\cline{1-2}
		Metropolized HMC \& MALA \cite{ChenDWY19} & $ \kappa^{1.5}\sqrt{d} + \kappa d$ & \\
		\cline{1-2}
		Metropolized HMC \& MALA (This paper)& $\kappa d$ &\\ 
		\hline 
	\end{tabular}
	\caption{Mixing times for algorithms in the condition number regime of logconcave sampling.}
	\label{tb: mixing_time}
\end{table}

\subsection{Our contribution}

Towards improving our understanding of Question~\ref{q:hmc}, we show that there is an algorithm which runs Metropolized HMC (defined in Algorithm~\ref{alg:mhmc}) for $O(\kappa d \log^3(\kappa d/\eps))$ iterations\footnote{The precise statement of our algorithmic guarantee can be found as Theorem~\ref{thm:mainclaim} of our paper.}, for sampling from a density $\exp(-f(x))$ defined on $\R^d$, where $f$ has a condition number of $\kappa$, and produces a point from a distribution with total variation at most $\eps$ away from the target density, for any $\eps > 0$. This is the first mixing-time guarantee for any algorithm in the high-accuracy regime accessing first-order function information from the log-density $f$ attaining linear dependence on the condition number $\kappa$, without additional smoothness assumptions (i.e. higher-order derivative bounds). Our mixing time bound improves upon a recent bound attaining linear dependence on $\kappa$ due to \cite{DurmusMM19}, of $\tilde{O}(\kappa d^2/\eps^4)$, in all parameters. Moreover, our dependence on the dimension $d$ matches the prior state-of-the-art \cite{DwivediCWY18, ChenDWY19}, and our algorithm does not require a warm start, as it explicitly bounds warmness dependence from a known starting distribution. 

In Section~\ref{sec:lowerbound}, we also give preliminary evidence that for the equivalent first-order sampling methods of Metropolized HMC and Metropolis-adjusted Langevin dynamics, $\Omega(\kappa)$ iterations are necessary even for sampling from a quadratic. In particular, we show that if the step size is not sufficiently small, the chain can only move with exponentially small probability, which we combine with a lower bound on the mixing time of continuous-time HMC in \cite{chen2019optimal} for small step sizes.

The starting point of our analysis is the mixing time analysis framework for the HMC algorithm in \cite{DwivediCWY18, ChenDWY19}. However, we introduce several technical modifications to overcome barriers in their work to obtain our improved mixing time bound, which we now discuss. We hope these tools may be of broader interest to both the community studying first-order sampling methods in the smooth, strongly logconcave regime, and sampling researchers in general.

\subsubsection{Gradient concentration}

How large is the norm of the gradient of a ``typical'' point drawn from the density $\exp(-f)$? It has been observed in a variety of recent works studying sampling algorithms \cite{lee2018algorithmic, shen2019randomized, VempalaW19} that the \emph{average} gradient norm of a point drawn from the target density is bounded by $\sqrt{Ld}$, where $L$ is the smoothness parameter of the function $f$ and $d$ is the ambient dimension; this observation has been used in obtaining state-of-the-art sampling algorithms in the $\textup{poly}(\eps^{-1})$ runtime regime. However, for runtimes obtaining a $\textrm{polylog}(\eps^{-1})$ runtime, this guarantee is not good enough, as it must hold for all points in a set of substantially larger measure than guaranteed by e.g. Markov's inequality. The weaker high-probability guarantee that the gradient norm is bounded by $\sqrt{Ld} \cdot \sqrt{\kappa}$ follows directly from sub-Gaussian concentration on the point $x$, and a Lipschitz guarantee on the gradient norm. Indeed, this weaker bound is the bottleneck term in the analysis of \cite{DwivediCWY18, ChenDWY19}, and prevents a faster algorithm when $\kappa > d$. Can we improve upon the average-case guarantee more generally when the log density $f$ is smooth?

For quadratic $f$, it is easy to see that the average gradient norm bound can be converted into a high-probability guarantee. We show that a similar concentration guarantee holds for \emph{all} logsmooth, strongly logconcave densities, which is the starting point of our improved mixing time bound. Our concentration proof follows straightforwardly from a Hessian-weighted variant of the well-known Poincar\'e inequality, combined with a reduction due to Herbst, as explored in \cite{Ledoux99}.

\subsubsection{Mixing time analysis}

The study of Markov chains producing iterates $\{x_k\}$, where the transition $x_k \rightarrow x_{k + 1}$ is described by an algorithm whose steady-state is a stationary distribution $\pis$, and $x_0$ is drawn from an initial distribution $\pi_0$, primarily focuses on characterizing the rate at which the distribution of iterates of the chain approaches $\pis$. To obtain a mixing time bound, i.e. a bound on the number of iterations needed for our algorithm to obtain a distribution within total variation distance $\eps$ of the stationary $\pis$, we follow the general framework of bounding the \emph{conductance} of the random walk defined by Metropolized HMC, initiated in a variety of works on Markov chain mixing times (e.g. \cite{SinclairJ89, lovasz1993random}). In particular, \cite{lovasz1993random} showed how to use the generalized notion of $s$-conductance to account for a small-probability ``bad'' region with poor random walk behavior. In our work, the ``good'' region $\Omega$ will be the set of points whose gradient has small norm. However, our mixing time analysis requires several modifications from prior work to overcome subtle technical issues. 

\paragraph{Average conductance.} As in prior work \cite{ChenDWY19}, because of the exponential warmness $\kappa^{d/2}$ of the starting distribution used, we require extensions in the theory of \emph{average conductance} \cite{LovaszK99} to obtain a milder dependence on the warmness, i.e. doubly logarithmic rather than singly logarithmic, to prevent an additional dimension dependence in the mixing time. The paper \cite{ChenDWY19} obtained this improved dependence on the warmness by generalizing the analysis of \cite{GoelMT06} to continuous-time walks and restrictions to high-probability regions. This analysis becomes problematic in our setting, as our region $\Omega$ may be nonconvex, and the restriction of a strongly logconcave function to a nonconvex set is possibly not even logconcave. This causes difficulties when bounding standard conductance notions which may depend on sets of small measure, because these sets may behave poorly under restriction by $\Omega$ (e.g. in the proof of Lemma~\ref{lem:conductance}).

\paragraph{Blocking conductance.} To mitigate the difficulty of poor small-set conductance due to the nonconvexity of $\Omega$, we use the \emph{blocking conductance} analysis of \cite{KannanLM06}, which averages conductance bounds of sets with measure \emph{equal to} some specified values in a range lower-bounded by roughly the inverse-warmness. In our case, this is potentially problematic, as the set where our concentration result guarantees that the norm of the gradient is not much larger than its mean has measure roughly $1 - \exp(-\sqrt{d})$, which is too small to bound the behavior of sets of size $\kappa^{-d/2}$ required by the quality of the warm start. However, we show that, perhaps surprisingly, the analysis of the blocking conductance is not bottlenecked by the worse quality of the gradient concentration required. In particular, the $\kappa^{1.5} \sqrt{d}$ runtime of \cite{DwivediCWY18, ChenDWY19} resulted from the statement, with probability at most $\exp(-d)$, the gradient norm is bounded by $\sqrt{L\kappa d}$. We are able to sharpen this by Corollary~\ref{corr:gradcon} to $\sqrt{L}d$, trading off a $\kappa$ for a $d$, which is sufficient for our tighter runtime.

\paragraph{Boosting to high accuracy.} Finally, the blocking conductance analysis of \cite{KannanLM06} makes an algorithmic modification. In particular, letting $d\pi_k$ be the density after running $k$ steps of the Markov chain from $\pi_0$, the analysis of \cite{KannanLM06} is able to guarantee that the \emph{average} density $d\rho_k \defeq \frac{1}{k}\sum_{0 \le i < k}d\pi_i$ converges to $d\pis$ at a rate roughly $1/k$, with a factor depending on the average conductance\footnote{We note averaging has been observed to improve sampling accuracy in a different setting \cite{DurmusMM19}; we leave as an interesting open direction whether this averaging is necessary for our method.}. In our case, we can show that in roughly $O(\kappa d)$ iterations of Algorithm~\ref{alg:mhmc}, the distance $\tvd{\rho_k}{\pis}$ is bounded by a constant. However, as the analysis requires averaging with a potentially poor starting distribution, it is not straightforward to obtain a rate of convergence with dependence $\log \eps^{-1}$ for potentially small values of $\eps$, rather than the $\eps^{-1}$ dependence typical of $1/k$ rates. Moreover, it is unclear in our setting how to apply standard arguments \cite{AldousD86, LovaszW95} which convert mixing time guarantees for obtaining a constant total variation distance to guarantees for total variation distance $\eps$ with a logarithmic overhead on $\eps$, because the definition of mixing time used is a worst-case notion over all starting points. We propose an alternative reduction based on mixing-time guarantees over arbitrary starting distributions of a specified warmness, which we use to boost our constant-accuracy mixing-time guarantee (see Appendix~\ref{app:reduction} for a more formal treatment). While it is simple and inspired by classical coupling-based reduction arguments, to the best of our knowledge this reduction is new in the literature, and may be of independent interest.

\subsection{Discussion}
While we obtain an improved upper bound on the runtime of Metropolized HMC, there are many questions which remain unanswered, and we believe are exciting to explore in the space of sampling algorithms and complexity. We state two here which we believe are natural.

\begin{question}
	\label{q:hmcl}
	Can we obtain a matching lower bound, or an improved upper bound, on the complexity of the Metropolized HMC algorithm in terms of the dependence on the dimension $d$?
\end{question}

\begin{question}
	\label{q:firstorder}
	Can we obtain improved bounds (upper or lower) for the complexity of first-order sampling algorithms in general, in the condition number regime? For example, is $\Omega(\kappa)$ iterations always necessary, implying a separation with the optimization setting?
\end{question}

\section{Preliminaries}
\label{sec:prelims}

\subsection{Notation}
\label{ssec:notation}

We denote the set $1 \le i \le d$ by $[d]$. For $S \subseteq \R^d$, $S^c$ is its complement $\R^d \setminus S$. $\norm{\cdot}$ is the $\ell_2$ norm ($\norm{x}^2 = \sum_{i \in [d]}x_i^2$ for $x \in \R^d$). Differentiable $f: \R^d \rightarrow \R$ is $\mu$-strongly convex and $L$-smooth if
\[f(x) + \inprod{\nabla f(x)}{y - x} + \frac{\mu}{2}\norm{y - x}^2 \le f(y) \le f(x) + \inprod{\nabla f(x)}{y - x} + \frac{L}{2}\norm{y - x}^2,\; \forall x, y \in \R^d.\]
It is well-known that smoothness is equivalent to having a Lipschitz gradient, i.e. $\norm{\nabla f(x) - \nabla f(y)} \le L\norm{x - y}$, and when $f$ is twice-differentiable, smoothness and strong convexity imply 
\[\mu I_d \preceq \nabla^2 f(x) \preceq LI_d\]
everywhere, where $I_d$ is the identity and $\preceq$ is the Loewner order. In this paper, function $f: \R^d \rightarrow \R$ will always be differentiable, $L$-smooth, and $\mu$-strongly convex, with minimizer $x^*$. We let $\kappa \defeq L/\mu \ge 1$ be the \emph{condition number} of $f$. We define the Hamiltonian $\ham$ of $(x, v) \in \R^d \times \R^d$ by
\[\ham(x, v) = f(x) + \half\norm{v}^2.\]
$\Nor(\mu, \Sigma)$ is the Gaussian density centered at a point $\mu \in \R^d$ with covariance matrix $\Sigma \in \R^{d\times d}$. For $A \subseteq \R^d$ and a distribution $\pi$, we write
\[\pi(A) \defeq \int_{x \in A} d\pi(x).\] 
We fix the definition of the distribution density $d\pis(x)$, where $d\pis(x)/dx \propto \exp(-f(x))$ has
\[d\pis(x) = \frac{\exp(-f(x))dx}{\int_{\R^d} \exp(-f(y)) dy},\; \int_{\R^d} d\pis(x) = 1.\]
The marginal in the first argument of the density on $\R^d \times \R^d$ proportional to $\exp(-\ham(x, v))$ is $d\pis$; we overload $d\pis(x, v)$ to mean this density. For distributions $\rho$, $\pi$ on $\R^d$, the total variation is
\[\tvd{\rho}{\pi} \defeq \sup_{A \subseteq \R^d} |\rho(A) - \pi(A)| = \half \int_{\R^d} \left|\frac{d\rho}{d\pi}(x) - 1\right|d\pi(x).\]
We say that a distribution $\pi$ is $\beta$-warm with respect to another distribution $\rho$ if 
\[\sup_{x \in \R^d} \frac{d\pi}{d\rho}(x) \le \beta.\]
We define the expectation and variance with respect to a distribution in the usual way:
\[\E_\pi[g] \defeq \int_{\R^d} g(x) d\pi(x),\; \Var_{\pi}[g] \defeq \E_\pi[g^2] - \left(\E_\pi[g]\right)^2. \]
Finally, to simplify some calculations, we assume that $d$ is bounded below by a small constant. In the absence of this bound, general-purpose mixing times for logconcave functions with no dependence on $\kappa$ attain our stated guarantees.

\subsection{Algorithm}

We state the Metropolized HMC algorithm to be analyzed throughout the remainder of this paper. We remark that it may be thought of as a symplectic discretization of the continuous-time Hamiltonian dynamics for $\ham(x, v) = f(x) + \half\norm{v}^2$,
\[\frac{dx}{dt} = \frac{\partial\ham(x, v)}{\partial v} = v,\; \frac{dv}{dt} = -\frac{\partial \ham(x, v)}{\partial x} = -\nabla f(x).\]
The HMC process can be thought of as a dual velocity $v$ accumulating the gradient of the primal point $x$, with the primal point being guided by the velocity, similar to the classical mirror descent algorithm. The algorithm resamples $v$ each timestep to attain the correct stationary distribution.
\begin{algorithm}[ht!]\caption{Metropolized Hamiltonian Monte Carlo: \texttt{HMC}$(\eta, x_0, f)$}
	\label{alg:mhmc}
	\textbf{Input:} Initial point $x_0 \in \R^d$, step size $\eta$. \\
	\textbf{Output:} Sequence $\{x_k\}$, $k \ge 0$.
	\begin{algorithmic}[1]
		\For{$k \ge 0$}
		\State Draw $v_k \sim \Nor(0, I_d)$. 
		\State $(\tx_k, \tv_k) \gets \lf(\eta, x_k, v_k)$.
		\State Draw $u$ uniformly in $[0, 1]$.
		\If{$u \le \min\left\{1, \exp(\ham(x, v) - \ham(\tx, \tv))\right\}$}
		\State $x_{k + 1} \gets \tx_k$.
		\Else 
		\State $x_{k + 1} \gets x_k$.
		\EndIf
		\EndFor
	\end{algorithmic}
\end{algorithm}

\begin{algorithm}[ht!]\caption{Leapfrog: \texttt{Leapfrog}$(\eta, x, v)$}
	\label{alg:LF}
	\textbf{Input:} Points $x, v \in \R^d$, step size $\eta$. \\
	\textbf{Output:} Points $\tx, \tv \in \R^d$.
	\begin{algorithmic}[1]
		\State $v' \gets v - \frac{\eta}{2}\nabla f(x)$.
		\State $\tx \gets x + \eta v'$.
		\State $\tv \gets v' - \frac{\eta}{2}\nabla f(\tx)$.
	\end{algorithmic}
\end{algorithm}

From a point $x \in \R^d$, we define $\prop_x$ to be the distribution of $\tx_k$ after one step of Algorithm~\ref{alg:mhmc} starting from $x_k = x$. Similarly, $\tran_x$ is the distribution of $x_{k + 1}$ starting at $x_k = x$, i.e. after the accept-reject step. Algorithm~\ref{alg:mhmc} uses the subprocedure $\lf$, which enjoys the following property.

\begin{lemma}
\label{lem:lfinv}
If $\lf(\eta, x, -v) = (\tx, \tv)$, then $\lf(\eta, \tx, -\tv) = (x, v)$.
\end{lemma}
\begin{proof}
Recall that $\lf(x, -v) = (\tx, \tv)$ implies
\[\tv = -v - \frac{\eta}{2}\nabla f(x) - \frac{\eta}{2}\nabla f(\tx),\; \tx = x - \eta v - \frac{\eta^2}{2}\nabla f(x).\]
Reversing these definitions yields the claim.
\end{proof}
\begin{corollary}
$d\pis$ is a stationary distribution for the Markov chain defined by Algorithm~\ref{alg:mhmc}.
\end{corollary}
\begin{proof}
We show that for $z = (x, v)$, $d\pis(z)/dz \propto \ham(z)$ is the stationary distribution on $(x_k, v_k)$; correctness then follows from $\pis$ having the correct marginal. Stationarity follows if and only if 
\[d\pis(x, v) \tran_{x, v}(\tx, \tv) = d\pis(x, v) \tran_{\tx, \tv}(x, v)\]
for all pairs $(x, v)$, $(\tx, \tv)$, where we overload the definition of $\tran$ to be the transition distribution from a point $(x, v)$. By the standard proof of correctness for the Metropolis-Hastings correction, i.e. choosing an acceptance probability proportional to 
\[\min\left\{1, \frac{d\pis(\tx, \tv)\prop_{\tx, \tv}(x, v)}{d\pis(x, v)\prop_{x, v}(\tx, \tv)}\right\},\]
it suffices to show that $\prop_{\tx, \tv}(x, v) = \prop_{x, v}(\tx, \tv)$. Note that $\prop_{x, v}$ is a deterministic proposal, and uniquely maps to a point $(\tx, \tv)$. Moreover, by symmetry of $\ham$ in the second argument, iteration $k$ of Algorithm~\ref{alg:mhmc} is equivalent to drawing $v_k$, negating it, and then running $\lf$. Correctness for this equivalent algorithm follows by Lemma~\ref{lem:lfinv}.
\end{proof}


\section{Gradient concentration}
\label{sec:concentration}

In this section, we give a bound on how well the norm of the gradient $\norm{\nabla f(x)}$ concentrates when $f$ is smooth and $x \sim d\pis(x)/dx \propto \exp(-f(x))$. First, we recall the following ``Hessian-weighted'' variant of the Poincar\'e inequality, which first appeared in \cite{BrascampL76}. 

\begin{theorem}[Hessian Poincar\'e]
\label{thm:hesspoin}
For probability density $d\pis(x)/dx \propto \exp(-f(x))$, and continuously differentiable function $g: \R^d \rightarrow \R$ with bounded variance with respect to $\pis$,
\[\Var_{\pis}[g] \le \int_{\R^d} \inprod{\left(\nabla^2 f(x)\right)^{-1}\nabla g(x)}{\nabla g(x)} d\pis(x).\]
\end{theorem}

An immediate corollary of Theorem~\ref{thm:hesspoin} is that the Poincar\'e constant of a $\mu$-strongly logconcave distribution is at most $\mu^{-1}$. While it does not appear to have been previously stated in the literature, our concentration bound can be viewed as a simple application of an argument of Herbst which reduces concentration to an isoperimetric inequality such as Theorem~\ref{thm:hesspoin}; an exposition of this technique can be found in \cite{Ledoux99}. We now state the concentration result.

\begin{theorem}[Gradient norm concentration]
\label{thm:gradcon}
If twice-differentiable $f: \R^d \rightarrow \R$ is $L$-smooth and convex, then for $d\pis(x)/dx \propto \exp(-f(x))$, and all $c > 0$,
\[\Pr_{\pis}\left[\norm{\nabla f(x)} \ge \epis\left[\norm{\nabla f}\right] + c\sqrt{L}\log d\right] \le 3d^{-c}.\]
\end{theorem}
\begin{proof}
Let $G(x) \defeq \norm{\nabla f(x)}$, and let $g(x) \defeq \exp(\half \lambda G(x))$. Clearly $g$ is continuously differentiable. Moreover, suppose first for simplicity that $f$ is strongly convex; then the existence of the variance of $g$ follows from the well-known fact that $f$ has sub-Gaussian tails (e.g.\ \cite{DwivediCWY18}, Lemma 1) and Lipschitzness of its gradient, from which the sublevel sets of the gradient norm grow more slowly than the decay of $\norm{x - x^*}_2$. The final conclusion has no dependence on the strong concavity of $f$, and we can extend this to arbitrary convex functions by regularizing by a small amount of quadratic regularizer (which only affects smoothness) and taking a limit as the regularizer amount vanishes. We now apply Theorem~\ref{thm:hesspoin}, which implies (noting that the gradient of $\norm{\nabla f}$ is $\nabla^2 f \tfrac{\nabla f}{\norm{\nabla f}}$)
\begin{align*}\epis\left[\exp(\lambda G)\right] - \epis\left[\exp\left(\frac{\lambda G}{2}\right)\right]^2 &\le \frac{\lambda^2}{4}\epis\left[\inprod{(\nabla^2 f)\frac{\nabla f}{\norm{\nabla f}}}{\frac{\nabla f}{\norm{\nabla f}}}\exp(\lambda G)\right] \\
&\le \frac{L\lambda^2}{4}\epis\left[\exp(\lambda G)\right].\end{align*}
In the last inequality we used smoothness. Letting $H(\lambda) \defeq \epis\left[\exp(\lambda G)\right]$, for $\lambda < \frac{2}{\sqrt{L}}$,
\[H(\lambda) \le \frac{1}{1 - \frac{L\lambda^2}{4}}H\left(\frac{\lambda}{2}\right)^2. \]
Using this recursively, we have
\[H(\lambda) \le \prod_{k = 0}^\infty \left(\frac{1}{1 - \frac{L\lambda^2}{4^{k + 1}}} \right)^{2^k} \lim_{\ell \rightarrow \infty} H\left(\frac{\lambda}{\ell}\right)^\ell. \]
There are two things to estimate on the right hand side. First, for sufficiently large $\ell$,
\[\epis\left[\exp\left(\frac{\lambda G}{\ell}\right)\right]^\ell \approx \left(1 + \epis\left[\frac{\lambda G}{\ell}\right]\right)^\ell \approx \exp\left(\lambda \epis\left[G\right]\right). \]
Second, letting $C = \frac{L\lambda^2}{4} < 1$, \cite{BobkovL97} showed that
\[\prod_{k = 0}^\infty \left(\frac{1}{1 - \frac{C}{4^k}} \right)^{2^k} \le \frac{1 + \sqrt{C}}{1 - \sqrt{C}}.\]
For completeness, we show this in Appendix~\ref{app:deferred}. Altogether, we have that for all $\lambda < \frac{2}{\sqrt{L}}$,
\[\epis\left[\exp(\lambda G)\right] \le \frac{1 + \half \sqrt{L}\lambda}{1 - \half \sqrt{L}\lambda}\exp\left(\lambda \epis\left[G\right]\right). \]
By Markov's inequality on the exponential, we thus conclude that
\[\Pr_{\pis}[G \ge \epis\left[G\right] + r] \le \exp(-\lambda r) \frac{1 + \half \sqrt{L}\lambda}{1 - \half \sqrt{L}\lambda}.  \]
Finally, letting $\lambda = \frac{1}{\sqrt{L}}$ and $r = c\sqrt{L}\log d$,
\[\Pr_{\pis}\left[\norm{\nabla f} \geq \sqrt{Ld} + c\sqrt{L}\log d\right] \le 3d^{-c}.\]	
\end{proof}

As an immediate corollary, we obtain the following.

\begin{corollary}
\label{corr:gradcon}
If twice-differentiable $f: \R^d \rightarrow \R$ is $L$-smooth and strongly convex, then $\forall c > 0$,
\[\Pr_{\pis}\left[\norm{\nabla f} \ge \sqrt{Ld} + c\sqrt{L}\log d\right] \le 3d^{-c}.\]
\end{corollary}
\begin{proof}
It suffices to show that
\begin{equation}\label{eq:meangrad}\epis\left[\norm{\nabla f}\right] \le \sqrt{Ld}.\end{equation}
This was observed in \cite{Dalalyan17b, VempalaW19}; we adapt a proof here. Observe that because 
\[\nabla \cdot(\nabla f(x) \pis(x)) = \Delta f(x) \pis(x) - \inprod{\nabla f(x)}{\nabla f(x)}\pi(x), \]
where $\nabla \cdot$ is divergence and $\Delta$ is the Laplacian operator, integrating both sides and noting that the boundary term vanishes,
\[\epis\left[\norm{\nabla f}^2\right] = \epis\left[\Delta f\right] \le Ld. \]
In the last equality, we used smoothness of $f$. \eqref{eq:meangrad} then follows from concavity of the square root.
\end{proof}

We remark that for densities $d\pis$ where a log-Sobolev variant of the inequality in Theorem~\ref{thm:hesspoin} holds, we can sharpen the bound in Corollary~\ref{corr:gradcon} to $O(d^{-c^2})$; we provide details in Appendix~\ref{app:logsob}. This sharpening is desirable for reasons related to the warmness of starting distributions for sampling from $\pis$, as will become clear in Section~\ref{sec:warmstart}. However, the ``Hessian log-Sobolev'' inequality is strictly stronger than Theorem~\ref{thm:hesspoin}, and does not hold for general strongly logconcave distributions \cite{BobkovL00}. Correspondingly, the concentration arguments derivable from Poincar\'e inequalities appear to be weaker \cite{Ledoux99}: we find exploring the tightness of Corollary~\ref{corr:gradcon} to be an interesting open question.


\section{Mixing time bounds via blocking conductance}
\label{sec:warmstart}

We first give a well-known bound of the warmness of an initial distribution; this starting distribution also was used in prior work in this setting \cite{Dalalyan17b, DwivediCWY18}

\begin{lemma}[Initial warmness]
\label{lem:betadef}
For $d\pis \propto \exp(-f(x))dx$ where $f$ is $L$-smooth and $\mu$-strongly convex with minimizer $x^*$\footnote{We remark that the minimizer $x^*$ can be efficiently found using e.g. an accelerated gradient method, to a degree of accuracy which does not bottleneck the runtime of Metropolized HMC by more than mild logarithmic factors. We defer a discussion of performance under inexact knowledge of the parameters $x^*, L$ to \cite{DwivediCWY18}, and assume their exact knowledge for simplicity in this work.}, $\pi_0 = \Nor(x^*, L^{-1}I_d)$ is a $\kappa^{d/2}$-warm distribution with respect to $\pis$.
\end{lemma}
\begin{proof}
By smoothness and strong convexity, and the density of a Gaussian distribution,
\[d\pi_0(x) = \frac{\exp\left(-\frac{L}{2}\norm{x - x^*}^2\right)dx}{(2\pi L^{-1})^{d/2}},\; d\pis(x) = \frac{\exp(-f(x))dx}{\int_{\R^d} \exp(-f(y)) dy} \ge \frac{\exp\left(-\frac{L}{2}\norm{x - x^*}^2\right)dx}{(2\pi \mu^{-1})^{d/2}}.\]
In the last inequality we normalized by $\exp(-f(x^*))$. Combining these bounds yields the result.
\end{proof}

Let $d\pi_k$ be the density of $x_k$ after running $k$ steps of Algorithm~\ref{alg:mhmc}, where $x_0$ is drawn from $\pi_0 = \Nor(x^*, L^{-1} I_d)$. Moreover, let $d\rho_k \defeq \frac{1}{k}\sum_{0 \le i < k} d\pi_i$ be the average density over the first $k$ iterations. In Section~\ref{ssec:constant}, we will show how to use the blocking conductance framework of \cite{KannanLM06} to obtain a bound on the number of iterations $k$ required to obtain a constant-accuracy distribution. We then show in Section~\ref{ssec:boost} that we can boost this guarantee to obtain total variation $\eps$ for arbitrary $\eps > 0$ with logarithmic overhead, resulting in our main mixing time claim, Theorem~\ref{thm:mainclaim}.

\subsection{Constant-accuracy mixing}
\label{ssec:constant}
We state the results required to prove a mixing-time bound for constant levels of total variation from the stationary measure $\pis$. All proofs are deferred to Appendix~\ref{app:mixing}. The first result is a restatement of the main result of \cite{KannanLM06}, modified for our purposes; recall that $\rho_k$ is an average over the distributions $\pi_i$ for $0 \le i < k$. Finally, we define $Q(S) \defeq \int_{S} \tran_x(S^c)d\pis(x)$ to be the probability one step of the walk starting at random point in a set $S$ leaves the set.

\begin{theorem}[Blocking conductance mixing bound]
\label{thm:blockmix}
	Suppose the starting distribution $\pi_0$ is $\beta$-warm with respect to $\pis$. Moreover, suppose for some $c$, and for all $c \le t \le \half$, we have a bound
	\begin{equation}\label{eq:phidef}\frac{\pis(S)}{Q(S)^2} \le \phi(t), \text{ for all } S \subseteq \R^d \text{ with } \pis(S) = t,\end{equation}
	for a decreasing function $\phi$ on the range $[c, \frac{1}{4}]$, and $\phi(x) \le M$ for $x \in [\frac{1}{4}, \half]$. Then, 
	\[\norm{\rho_k - \pis}_{\textup{TV}} \le \beta c + \frac{32}{k}\left(\int_{c}^{1/4} \phi(x) dx + \frac{M}{4} \right).\]
\end{theorem}

At a high level, the mixing time requires us to choose a threshold $c$ which is inversely-proportional to the warmness, and bound the average value of a function $\phi(t)$ in the range $[c, \half]$, where $\phi(t)$ serves as an indicator of how ``bottlenecked'' sets of measure exactly equal to $t$ are.

Next, by using a logarithmic isoperimetric inequality from \cite{ChenDWY19}, we show in the following lemma that we can bound $\frac{\pi^*(S)}{Q(S)^2}$ when $\pi^*(S)$ is in some range.

\begin{lemma}
	\label{lem:conductancesimple}
	Suppose for $\Omega \subset R^d$ with $\pis(\Omega) = 1 - s$, and all $x, y \in \Omega$ with $\norm{x - y} \le \eta$, 
	\begin{equation}\label{eq:neededtv}\norm{\tran_x - \tran_y}_{\textup{TV}} \le 1 - \alpha.\end{equation}
	Then, if $\pis$ is $\mu$-strongly logconcave, $\eta\sqrt{\mu} < 1$, and $s \le \frac{\eta\sqrt{\mu}t}{16}$, for all $t \le \half$,
	\[\frac{\pis(S)}{Q(S)^2} \le \frac{2^{16}}{\alpha^2\eta^2\mu t\log(1/t)}, \; \forall S \text{ with } \pis(S) = t.\]
\end{lemma}

For a more formal statement and proof, see Lemma~\ref{lem:conductance}. Note that in particular the lower range of $t$ required by Theorem~\ref{thm:blockmix} is at least inversely proportional to the warmness, which causes the gradient norm bound obtained by the high-probability set $\Omega$ to lose roughly a $\sqrt{d}$ factor. To this end, for a fixed positive $\eps \le 1$, denote
\begin{equation}\label{eq:omegadef}\Omega \defeq \left\{x \in \R^d \mid \norm{\nabla f(x)}_2 \le 5\sqrt{L}d\log\frac{\kappa}{\eps}\right\}.\end{equation}
In Appendix~\ref{app:totalvariation}, we show the following.

\begin{lemma}
	\label{lem:tvst}
	For $\eta^2 \le \frac{1}{20Ld\log\frac{\kappa}{\eps}}$ and all $x$, $y \in \Omega$ with $\norm{x - y} \le \eta$,
	\[\norm{\tran_x - \tran_y}_{\textup{TV}} \le \frac{7}{8}.\]
\end{lemma}

By combining these pieces, we are able to obtain an algorithm which mixes to constant total variation distance from the stationary distribution $\pis$ in $\tilde{O}(\kappa d)$ iterations.

\begin{proposition}
\label{prop:warmeps}
Let $\eps \in [0, 1]$, $\beta = \kappa^{d/2}$. From any $\beta/\eps$-warm initial distribution $\pi_0$, running Algorithm~\ref{alg:mhmc} for $j$ iterations, where $j$ is uniform between $0$ and $k - 1$ for $k > C\kappa d\log\kappa\log\log\beta$ for universal constant $C$, returns from distribution $\rho_k$ with $\tvd{\rho_k}{\pis} < (2e)^{-1}$.
\end{proposition}
\begin{proof}
Note that $\rho_k$ as defined in the theorem statement is precisely the $\rho_k$ of Theorem~\ref{thm:blockmix}. Moreover, for the set $\Omega$ in \eqref{eq:omegadef}, the probability $x \sim \pis$ is not in $\Omega$ is bounded via Corollary~\ref{corr:gradcon} by 
\begin{equation}\label{eq:sbound} s < 3d^{-5d\log_d(\kappa/\eps)} < \left(\kappa/\eps\right)^{-4d}.\end{equation}
For $\eta = \sqrt{\frac{1}{20Ld\log\frac{\kappa}{\eps}}}$ and $c \defeq \eps/(4\beta e)$, $s\leq \frac{\eta\sqrt{\mu}t}{16}$ is satisfied for all $t$ in the range $[c, \half]$.  Thus, we can apply Lemma~\ref{lem:conductancesimple} and conclude that \eqref{eq:phidef} holds for the function
\[\phi(t) = \left(20 \cdot 2^{22}\kappa d\log\frac{\kappa}{\eps}\right)\frac{1}{t\log(1/t)}.\]
Next, note that $\phi(t)$ is decreasing in the range $[c, 1/e]$, and attains its maximum at $t = \half$ within the range $t \in [\frac{1}{4}, \half]$, where $2/\log 2 < 3$. Thus, the conditions of Theorem~\ref{thm:blockmix} apply, such that
\begin{align*}\tvd{\rho_k}{\pis} &\le \frac{1}{4e} + \frac{20 \cdot 2^{27} \kappa d\log\frac{\kappa}{\eps}}{k} \left(\int_{\eps/(4\beta e)}^{1/4} \frac{1}{x\log(1/x)}dx + \frac{3}{4}\right) \\
&\le \frac{1}{4e} + \frac{20 \cdot 2^{27} \kappa d\log\frac{\kappa}{\eps}}{k}\left(\log\log\left(\frac{4\beta e}{\eps}\right) - \log \log 4 + \frac{3}{4}\right).\end{align*}
Thus, by choosing $k$ to be a sufficiently large multiple of $\kappa d \log\frac{\kappa}{\eps} \log\log\frac{\beta}{\eps}$, the guarantee follows.
\end{proof}

\subsection{High-accuracy mixing}
\label{ssec:boost}
We now state a general framework for turning guarantees such as Proposition~\ref{prop:warmeps} into a $\eps$-accuracy mixing bound guarantee, with logarithmic overhead in the quantity $\eps$. We defer a more specific statement and proof to Appendix~\ref{app:reduction}.
\begin{lemma}
	Suppose there is a Markov chain with transitions given by $\widetilde{\tran}$, and some nonnegative integer $\tmix$, such that for every $\pi$ which is a $\beta/\eps$-warm distribution with respect to $\pis$,
	\begin{equation}\label{eq:tmix}\norm{\widetilde{\tran}^{\tmix}\pi - \pis}_{\textup{TV}} \le \frac{1}{2e}.\end{equation}
	Then, if $\pi_0$ is a $\beta$-warm start, and $k \ge \tmix\log(\eps^{-1})$,
	\[\norm{\widetilde{\tran}^k\pi_0 - \pis}_{\textup{TV}} \le \eps.\]
\end{lemma}
At a high level, the proof technique is to couple points according to the total variation bound between $\tran^i \pi_0$ and $\pis$ every $\tmix$ iterations, while the total variation distance is at least $\eps$. This in turn allows us to bound the warmness of the ``conditional distribution'' of uncoupled points by $\beta/\eps$ using the fact that the total variation bound measures the size of the set of uncoupled points, and use the guarantee \eqref{eq:tmix} iteratively. We can now state our main claim.

\begin{theorem}[Mixing of Hamiltonian Monte Carlo]
\label{thm:mainclaim}
There is an algorithm initialized from a point drawn from $\Nor(x^*, L^{-1}I_d)$, which iterates Algorithm~\ref{alg:mhmc} for 
\[O\left(\kappa d \log\left(\frac{\kappa}{\eps}\right)\log\left(d\log\frac{\kappa}{\eps}\right)\log\left(\frac{1}{\eps}\right)\right)\]
iterations, and produces a point from a distribution $\rho$ such that $\tvd{\rho}{\pis} \le \eps$.
\end{theorem}
\begin{proof}
Define a Markov chain with transitions $\widetilde{\tran}$, whose one-step distribution from an initial point is to run the algorithm of Proposition~\ref{prop:warmeps}. Note that each step of the Markov chain with transitions $\widetilde{\tran}$ requires $O\left(\kappa d \log\left(\frac{\kappa}{\eps}\right)\log\left(d\log\frac{\kappa}{\eps}\right)\right)$ iterations of Algorithm~\ref{alg:mhmc}, and the averaging step is easily implementable by sampling a random stopping time at uniform. Moreover, the Markov chain with transitions $\widetilde{\tran}$ satisfies \eqref{eq:tmix} with $\tmix = 1$, by the guarantees of Corollary~\ref{corr:gradcon}. Thus, by running $\log(\eps^{-1})$ iterations of this Markov chain, we obtain the required guarantee.
\end{proof}

	\section{Step-size lower bound}
\label{sec:lowerbound}

We make progress towards showing that standard Metropolized first-order methods (i.e. Algorithm~\ref{alg:mhmc}, which we recall is equivalent to the Metropolis-adjusted Langevin dynamics) have mixing times with at least a linear dependence on $\kappa$. Formally, consider the $d$-dimensional quadratic 
\[
f(x)=\frac{1}{2}x^\top Dx,
\]
where $D=\diag{\lambda_{1},\ldots,\lambda_{d}}$, and $\lambda_i=\kappa$ for $i \in [d - 1]$,
$\lambda_{d}=1$. This choice of $f$ is $1$-strongly convex and $\kappa$-smooth. We show that running Algorithm~\ref{alg:mhmc} with a step size much larger than $L^{-1/2}$, from a random point from the stationary distribution, will have an exponentially small accept probability. Moreover, \cite{chen2019optimal} shows that even the continuous-time HMC algorithm with step size $\eta = O(L^{-1/2})$ requires $\Omega\left(\kappa \right)$ iterations to converge.

\begin{theorem}[Theorem 4 of \cite{chen2019optimal}]
	\label{thm:lower_bound_mixing}For $L$-smooth, $\mu$-strongly convex $f$, the relaxation time of ideal HMC for sampling from the density $x\sim \exp(-f(x))$ with step size $\eta = O(1/\sqrt{L})$ is $\Omega(\kappa)$.
\end{theorem}

\begin{lemma}
	 In one step of Algorithm~\ref{alg:mhmc}, suppose the starting point $x \sim \exp(-f(x))$ is drawn from the stationary distribution. For any step size $\eta = c\kappa^{-1/2}$, for $c > 40$, and with probability at least $1-3d^{-25}$ over the randomness of $x$, $v$, the accept probability satisfies
	\[
	\exp\left(-\mathcal{H}(\tilde{x},\tilde{v})+\mathcal{H}(x,v)\right)\leq \exp(-\Omega(c^6 d)).
	\]
\end{lemma}

\begin{proof}
	Recalling $\tilde{v}=v-\frac{\eta}{2}Dx-\frac{\eta}{2}D\tx$ and $\tx = x + \eta v - \frac{\eta^2}{2}Dx$, we compute
	\begin{align*}
	\mathcal{H}(\tilde{x},\tilde{v})-\mathcal{H}(x,v) &= \frac{1}{2}\left\Vert \tilde{v}\right\Vert ^{2}-\frac{1}{2}\left\Vert v\right\Vert ^{2}+\half \tx^\top D \tx- \half x^\top D x\\
	&= \frac{1}{2}\left\Vert v-\frac{\eta}{2}Dx-\frac{\eta}{2}D\tx\right\Vert ^{2}-\frac{1}{2}\left\Vert v\right\Vert ^{2}+\frac{1}{2}\inprod{D(x + \tx)}{\tx - x}\\
	&= -\frac{\eta}{2}\inprod{v}{D(x + \tx)} + \frac{\eta^2}{8}\norm{D(x + \tx)}^2 + \half \inprod{D(x + \tx)}{\eta v - \frac{\eta^2}{2}Dx}\\
	&=  \frac{\eta^2}{8}\left(\tx^\top D^2 \tx - x^\top D^2 x\right).
	\end{align*}	
	Let $x_{i}$ be the $i^{th}$ coordinate of $x$. For $i\in[d - 1]$, $\lambda_i = \kappa$, and 
	\begin{align*}
	\tilde{x}_{i}^{2}-x_{i}^{2} & =2 x_{i}\left(\eta v_{i}-\frac{\eta^2}{2}\lambda_{i}x_{i}\right)+\left(\eta v_{i}-\frac{\eta^2}{2}\lambda_{i}x_{i}\right)^{2}\\
	& =2\eta x_{i}v_{i}-\eta^{2}\lambda_{i}x_{i}^{2}+\eta^{2}v_{i}^{2}+\frac{\eta^{4}}{4}\lambda_{i}^{2}x_{i}^{2}-\eta^{3}\lambda_{i}x_{i}v_{i}\\
	& \geq-\left(\eta^2 v_i^2 + x_i^2\right) - \eta^2\lambda_i x_i^2 + \eta^2 v_i^2 + \frac{\eta^4}{4}\lambda_i^2 x_i^2 - \left(\frac{\eta^4}{8}\lambda_i^2x_i^2 + 2\eta^2 v_i^2\right)\\
	&=\left(\frac{\eta^4}{8}\lambda_i^2 - \eta^2 \lambda_i - 1\right) x_i^2 - 2\eta^2 v_i^2 \ge  \frac{\eta^4}{16} \lambda_i^2 x_i^2 - 2\eta^2 v_i^2,
	\end{align*}
	where the last step follows from our assumption that $\eta^2 \kappa \ge 20$. Then, letting $\hat{x}$ be the vector containing the first $d - 1$ entries of the vector $x$,
	\begin{equation}
	\begin{aligned}
	\label{eq:hambound}
	\mathcal{H}(\tilde{x},\tilde{v})-\mathcal{H}(x,v) &\ge \frac{\eta^2}{8}\left( \kappa^2\sum_{i \in [d - 1]} \left(\frac{\eta^4}{16}\kappa^2x_i^2 - 2\eta^2 v_i^2\right) + \left(\frac{\eta^4}{8} - \eta^2 - 1\right)x_d^2 - 2\eta^2 v_d^2\right)\\
	&\ge \frac{\eta^6\kappa^4}{128}\norm{\hat{x}}^2 - \frac{\eta^4\kappa^2}{4}\norm{v}^2 + \frac{\eta^2}{8} \left(\frac{\eta^4}{8} - \eta^2 - 1\right)x_d^2.
	\end{aligned}
	\end{equation}
	By standard tail bounds on the chi-squared distribution (Lemma 1 of \cite{LaurentM00}), with probability at least $1 - 3d^{-25}$, all of the following hold:
	\[
	\kappa \norm{\hat{x}}^2 \ge (d - 1) - 10\sqrt{d\log d},\; \norm{v}^2 \ge d + 10\sqrt{d\log d} + 50\log d,\; x_d^2 \le 10\log d.
	\]
	Plugging these bounds into \eqref{eq:hambound}, noting that the dominant term $\eta^6 \kappa^3$ behaves as $c^6$, and using that if $\eta > 10$ then the third term is nonnegative, we see that for sufficiently large $d$, the probability the chain can move is at most $\exp(-\Omega(c^6 d))$.
\end{proof}
	
	\newpage
	\bibliographystyle{alpha}	
	\bibliography{kd-hmc}

\newcommand{\etalchar}[1]{$^{#1}$}
\begin{thebibliography}{MMW{\etalchar{+}}19}

\bibitem[Aba16]{Abadi16}
Mart{\'{\i}}n Abadi.
\newblock Tensorflow: learning functions at scale.
\newblock In {\em Proceedings of the 21st {ACM} {SIGPLAN} International
  Conference on Functional Programming, {ICFP} 2016, Nara, Japan, September
  18-22, 2016}, page~1, 2016.

\bibitem[AD86]{AldousD86}
David~J. Aldous and Persi Diaconis.
\newblock Shuffling cards and stopping times.
\newblock {\em American Mathematical Monthly}, 93:333--348, 1986.

\bibitem[AdFDJ03]{AndrieuFDJ03}
Christophe Andrieu, Nando de~Freitas, Arnaud Doucet, and Michael~I. Jordan.
\newblock An introduction to {MCMC} for machine learning.
\newblock {\em Machine Learning}, 50(1-2):5--43, 2003.

\bibitem[Bar12]{Barber12}
David Barber.
\newblock {\em Bayesian reasoning and machine learning}.
\newblock Cambridge University Press, 2012.

\bibitem[Bes94]{Besag94}
Julian Besag.
\newblock Comments on ``representations of knowledge in complex systems'' by u.
  grenander and mi miller.
\newblock {\em Journal of the Royal Statistical Society, Series B},
  56:591--592, 1994.

\bibitem[BGJM11]{BrooksGJM11}
Steve Brooks, Andrew Gelman, Galin Jones, and Xiao-Li Meng.
\newblock {\em Handbook of Markov Chain Monte Carlo}.
\newblock Chapman and Hall/CRC, 2011.

\bibitem[BL76]{BrascampL76}
Herm~Jan Brascamp and Elliott~H Lieb.
\newblock On extensions of the brunn-minkowski and prékopa-leindler theorems,
  including inequalities for log concave functions, and with an application to
  the diffusion equation.
\newblock {\em Journal of Functional Analysis}, 22(4):366--389, 1976.

\bibitem[BL97]{BobkovL97}
Sergey~G Bobkov and Michel Ledoux.
\newblock Poincar\'e's inequalities and talagrand's concentration phenomenon
  for the exponential distribution.
\newblock {\em Probability Theory and Related Fields}, 107(3):383--400, 1997.

\bibitem[BL00]{BobkovL00}
Sergey~G Bobkov and Michel Ledoux.
\newblock From brunn-minkowski to brascamp-lieb and to logarithmic sobolev
  inequalities.
\newblock {\em GAFA, Geometric and Functional Analysis}, 10:1028--1052, 2000.

\bibitem[BRH13]{bou2013nonasymptotic}
Nawaf Bou-Rabee and Martin Hairer.
\newblock Nonasymptotic mixing of the mala algorithm.
\newblock {\em IMA Journal of Numerical Analysis}, 33(1):80--110, 2013.

\bibitem[BRS93]{belisle1993hit}
Claude~JP B{\'e}lisle, H~Edwin Romeijn, and Robert~L Smith.
\newblock Hit-and-run algorithms for generating multivariate distributions.
\newblock {\em Mathematics of Operations Research}, 18(2):255--266, 1993.

\bibitem[Bub15]{Bubeck15}
S{\'{e}}bastien Bubeck.
\newblock Convex optimization: Algorithms and complexity.
\newblock {\em Foundations and Trends in Machine Learning}, 8(3-4):231--357,
  2015.

\bibitem[CCBJ18]{cheng2017underdamped}
Xiang Cheng, Niladri~S. Chatterji, Peter~L. Bartlett, and Michael~I. Jordan.
\newblock Underdamped langevin {MCMC:} {A} non-asymptotic analysis.
\newblock In {\em Conference On Learning Theory, {COLT} 2018, Stockholm,
  Sweden, 6-9 July 2018}, pages 300--323, 2018.

\bibitem[CDWY19]{ChenDWY19}
Yuansi Chen, Raaz Dwivedi, Martin~J. Wainwright, and Bin Yu.
\newblock Fast mixing of metropolized hamiltonian monte carlo: Benefits of
  multi-step gradients.
\newblock {\em CoRR}, abs/1905.12247, 2019.

\bibitem[CGH{\etalchar{+}}17]{CarpenterGHLG17}
Bob Carpenter, Andrew Gelman, Matthew~D. Hoffman, Daniel Lee, Ben Goodrich,
  Michael Betancourt, Marcus Brubaker, Jiqiang Guo, Peter Li, and Allen
  Riddell.
\newblock Stan: A probabilistic programming language.
\newblock {\em Journal of Statistical Software}, 76(1), 2017.

\bibitem[CV19]{chen2019optimal}
Zongchen Chen and Santosh~S Vempala.
\newblock Optimal convergence rate of hamiltonian monte carlo for strongly
  logconcave distributions.
\newblock {\em arXiv preprint arXiv:1905.02313}, 2019.

\bibitem[Dal17a]{Dalalyan17b}
Arnak~S. Dalalyan.
\newblock Further and stronger analogy between sampling and optimization:
  Langevin monte carlo and gradient descent.
\newblock In {\em Proceedings of the 30th Conference on Learning Theory, {COLT}
  2017, Amsterdam, The Netherlands, 7-10 July 2017}, pages 678--689, 2017.

\bibitem[Dal17b]{dalalyan2017theoretical}
Arnak~S Dalalyan.
\newblock Theoretical guarantees for approximate sampling from smooth and
  log-concave densities.
\newblock {\em Journal of the Royal Statistical Society: Series B (Statistical
  Methodology)}, 79(3):651--676, 2017.

\bibitem[DCWY18]{DwivediCWY18}
Raaz Dwivedi, Yuansi Chen, Martin~J. Wainwright, and Bin Yu.
\newblock Log-concave sampling: Metropolis-hastings algorithms are fast!
\newblock In {\em Conference On Learning Theory, {COLT} 2018, Stockholm,
  Sweden, 6-9 July 2018}, pages 793--797, 2018.

\bibitem[DM{\etalchar{+}}19]{durmus2019high}
Alain Durmus, Eric Moulines, et~al.
\newblock High-dimensional bayesian inference via the unadjusted langevin
  algorithm.
\newblock {\em Bernoulli}, 25(4A):2854--2882, 2019.

\bibitem[DMM19]{DurmusMM19}
Alain Durmus, Szymon Majewski, and Blazej Miasojedow.
\newblock Analysis of langevin monte carlo via convex optimization.
\newblock {\em J. Mach. Learn. Res.}, 20:73:1--73:46, 2019.

\bibitem[DRD18]{dalalyan2018sampling}
Arnak~S Dalalyan and Lionel Riou-Durand.
\newblock On sampling from a log-concave density using kinetic langevin
  diffusions.
\newblock {\em arXiv preprint arXiv:1807.09382}, 2018.

\bibitem[GM91]{gelfand1991recursive}
Saul~B Gelfand and Sanjoy~K Mitter.
\newblock Recursive stochastic algorithms for global optimization in r\^{}d.
\newblock {\em SIAM Journal on Control and Optimization}, 29(5):999--1018,
  1991.

\bibitem[GMT06]{GoelMT06}
Sharad Goel, Ravi Montenegro, and Prasad Tetali.
\newblock Mixing time bounds via the spectral profile.
\newblock {\em Electronic Journal of Probability}, 11:1--26, 2006.

\bibitem[KLM06]{KannanLM06}
Ravi Kannan, L{\'{a}}szl{\'{o}} Lov{\'{a}}sz, and Ravi Montenegro.
\newblock Blocking conductance and mixing in random walks.
\newblock {\em Combinatorics, Probability {\&} Computing}, 15(4):541--570,
  2006.

\bibitem[Kra40]{kramers1940brownian}
Hendrik~Anthony Kramers.
\newblock Brownian motion in a field of force and the diffusion model of
  chemical reactions.
\newblock {\em Physica}, 7(4):284--304, 1940.

\bibitem[Led99]{Ledoux99}
Michel Ledoux.
\newblock {\em Concentration of measure and logarithmic Sobolev inequalities}.
\newblock Seminaire de probabilities XXXIII, 1999.

\bibitem[LK99]{LovaszK99}
L{\'{a}}szl{\'{o}} Lov{\'{a}}sz and Ravi Kannan.
\newblock Faster mixing via average conductance.
\newblock In {\em Proceedings of the Thirty-First Annual {ACM} Symposium on
  Theory of Computing, May 1-4, 1999, Atlanta, Georgia, {USA}}, pages 282--287,
  1999.

\bibitem[LM00]{LaurentM00}
Béatrice Laurent and Pascal Massart.
\newblock Adaptive estimation of a quadratic functional by model selection.
\newblock {\em The Annals of Statistics}, 28(5):1302--1338, 2000.

\bibitem[Lov99]{lovasz1999hit}
L{\'a}szl{\'o} Lov{\'a}sz.
\newblock Hit-and-run mixes fast.
\newblock {\em Mathematical Programming}, 86(3):443--461, 1999.

\bibitem[LPW09]{LevinPW09}
David~Asher Levin, Yuval Peres, and Elizabeth Wilmer.
\newblock {\em Markov Chains and Mixing Times}.
\newblock American Mathematical Society, 2009.

\bibitem[LS90]{lovasz1990mixing}
L{\'a}szl{\'o} Lov{\'a}sz and Mikl{\'o}s Simonovits.
\newblock The mixing rate of markov chains, an isoperimetric inequality, and
  computing the volume.
\newblock In {\em Proceedings [1990] 31st annual symposium on foundations of
  computer science}, pages 346--354. IEEE, 1990.

\bibitem[LS93]{lovasz1993random}
L{\'a}szl{\'o} Lov{\'a}sz and Mikl{\'o}s Simonovits.
\newblock Random walks in a convex body and an improved volume algorithm.
\newblock {\em Random structures \& algorithms}, 4(4):359--412, 1993.

\bibitem[LSV18]{lee2018algorithmic}
Yin~Tat Lee, Zhao Song, and Santosh~S Vempala.
\newblock Algorithmic theory of odes and sampling from well-conditioned
  logconcave densities.
\newblock {\em arXiv preprint arXiv:1812.06243}, 2018.

\bibitem[LV06]{lovasz2006hit}
L{\'a}szl{\'o} Lov{\'a}sz and Santosh Vempala.
\newblock Hit-and-run from a corner.
\newblock {\em SIAM Journal on Computing}, 35(4):985--1005, 2006.

\bibitem[LW95]{LovaszW95}
L{\'{a}}szl{\'{o}} Lov{\'{a}}sz and Peter Winkler.
\newblock Efficient stopping rules for markov chains.
\newblock In {\em Proceedings of the Twenty-Seventh Annual {ACM} Symposium on
  Theory of Computing, 29 May-1 June 1995, Las Vegas, Nevada, {USA}}, pages
  76--82, 1995.

\bibitem[MCJ{\etalchar{+}}18]{MaCJFJ18}
Yi{-}An Ma, Yuansi Chen, Chi Jin, Nicolas Flammarion, and Michael~I. Jordan.
\newblock Sampling can be faster than optimization.
\newblock {\em CoRR}, abs/1811.08413, 2018.

\bibitem[MMW{\etalchar{+}}19]{mou2019high}
Wenlong Mou, Yi-An Ma, Martin~J Wainwright, Peter~L Bartlett, and Michael~I
  Jordan.
\newblock High-order langevin diffusion yields an accelerated mcmc algorithm.
\newblock {\em arXiv preprint arXiv:1908.10859}, 2019.

\bibitem[MS17]{mangoubi2017rapid}
Oren Mangoubi and Aaron Smith.
\newblock Rapid mixing of hamiltonian monte carlo on strongly log-concave
  distributions.
\newblock {\em arXiv preprint arXiv:1708.07114}, 2017.

\bibitem[MT{\etalchar{+}}96]{mengersen1996rates}
Kerrie~L Mengersen, Richard~L Tweedie, et~al.
\newblock Rates of convergence of the hastings and metropolis algorithms.
\newblock {\em The annals of Statistics}, 24(1):101--121, 1996.

\bibitem[N{\etalchar{+}}11]{neal2011mcmc}
Radford~M Neal et~al.
\newblock Mcmc using hamiltonian dynamics.
\newblock {\em Handbook of markov chain monte carlo}, 2(11):2, 2011.

\bibitem[Nes03]{Nesterov03}
Yurii Nesterov.
\newblock {\em Introductory Lectures on Convex Optimization: A Basic Course,
  volume I}.
\newblock 2003.

\bibitem[PST{\etalchar{+}}12]{pillai2012optimal}
Natesh~S Pillai, Andrew~M Stuart, Alexandre~H Thi{\'e}ry, et~al.
\newblock Optimal scaling and diffusion limits for the langevin algorithm in
  high dimensions.
\newblock {\em The Annals of Applied Probability}, 22(6):2320--2356, 2012.

\bibitem[RT96a]{roberts1996geometric}
Gareth~O Roberts and Richard~L Tweedie.
\newblock Geometric convergence and central limit theorems for multidimensional
  hastings and metropolis algorithms.
\newblock {\em Biometrika}, 83(1):95--110, 1996.

\bibitem[RT{\etalchar{+}}96b]{roberts1996exponential}
Gareth~O Roberts, Richard~L Tweedie, et~al.
\newblock Exponential convergence of langevin distributions and their discrete
  approximations.
\newblock {\em Bernoulli}, 2(4):341--363, 1996.

\bibitem[SJ89]{SinclairJ89}
Alistair Sinclair and Mark Jerrum.
\newblock Approximate counting, uniform generation and rapidly mixing markov
  chains.
\newblock {\em Inf. Comput.}, 82(1):93--133, 1989.

\bibitem[SL19]{shen2019randomized}
Ruoqi Shen and Yin~Tat Lee.
\newblock The randomized midpoint method for log-concave sampling.
\newblock In {\em Advances in Neural Information Processing Systems}, pages
  2098--2109, 2019.

\bibitem[Tal19]{Talwar19}
Kunal Talwar.
\newblock Computational separations between sampling and optimization.
\newblock In {\em Advances in Neural Information Processing Systems 32: Annual
  Conference on Neural Information Processing Systems 2019, NeurIPS 2019, 8-14
  December 2019, Vancouver, BC, Canada}, pages 14997--15007, 2019.

\bibitem[Vem05]{Vempala05}
Santosh Vempala.
\newblock Geometric random walks: A survey.
\newblock {\em MSRI Combinatorial and Computational Geometry}, 52:573--612,
  2005.

\bibitem[VW19]{VempalaW19}
Santosh~S. Vempala and Andre Wibisono.
\newblock Rapid convergence of the unadjusted langevin algorithm: Isoperimetry
  suffices.
\newblock In {\em Advances in Neural Information Processing Systems 32: Annual
  Conference on Neural Information Processing Systems 2019, NeurIPS 2019, 8-14
  December 2019, Vancouver, BC, Canada}, pages 8092--8104, 2019.

\bibitem[XSL{\etalchar{+}}14]{xifara2014langevin}
Tatiana Xifara, Chris Sherlock, Samuel Livingstone, Simon Byrne, and Mark
  Girolami.
\newblock Langevin diffusions and the metropolis-adjusted langevin algorithm.
\newblock {\em Statistics \& Probability Letters}, 91:14--19, 2014.

\end{thebibliography}
	\newpage
	\begin{appendix}

\section{Equivalence of HMC and Metropolis-adjusted Langevin dynamics}
\label{app:equivalence}

We briefly remark on the equivalence of Metropolized HMC and the Metropolis-adjusted Langevin dynamics algorithm (MALA), a well-studied algorithm since its introduction in \cite{Besag94}. This equivalence was also commented on in \cite{ChenDWY19}. The algorithm can be seen as a filtered discretization of the continuous-time Langevin dynamics,
\[dx_t = -\nabla f(x_t) dt + \sqrt{2}dW_t,\]
where $W_t$ is Brownian motion. In short, the Metropolized HMC update is
\[v \sim \Nor(0, I),\; \tx \gets x + \eta v - \frac{\eta^2}{2}\nabla f(x), \text{ accept with probability } \min\left\{1, \frac{\exp(-\ham(\tx, \tv))}{\exp(-\ham(x, v))}\right\}. \]
Similarly, the MALA update with step size $h$ is
\[\tx \sim \Nor(x - h\nabla f(x), 2hI), \text{ accept with probability } \min\left\{1, \frac{\exp(-f(\tx) - \norm{x - \tx + h\nabla f(\tx)}_2^2/4h)}{\exp(-f(x) - \norm{\tx - x + h\nabla f(x)}_2^2/4h)}\right\}.\]
It is clear that in HMC the distribution of $\tx$ is
\[\tx \sim \Nor\left(x - \frac{\eta^2}{2}\nabla f(x), \eta^2 I\right), \]
so it suffices to show for $h = \eta^2/2$,
\[\frac{\norm{\tx - x + h\nabla f(x)}_2^2 - \norm{x - \tx + h\nabla f(\tx)}_2^2}{4h} = \half\left(\norm{v}_2^2 - \norm{\tv}_2^2\right). \]
Indeed, the right hand side simplifies to
\[\frac{\eta}{2}\inprod{\nabla f(\tx) + \nabla f(x)}{v} - \frac{\eta^2}{8}\norm{\nabla f(\tx) + \nabla f(x)}_2^2, \]
and the left hand side is
\begin{align*}\half\inprod{\nabla f(\tx) + \nabla f(x)}{\tx - x} + \frac{h}{4}\left(\norm{\nabla f(x)}_2^2 - \norm{\nabla f(\tx)}_2^2\right) \\
= \half\inprod{\nabla f(\tx) + \nabla f(x)}{\eta v - \frac{\eta^2}{2}\nabla f(x)} + \frac{\eta^2}{8}\left(\norm{\nabla f(x)}_2^2 - \norm{\nabla f(\tx)}_2^2\right). \end{align*}
Comparing coefficients shows the equivalence.

\section{Improved concentration under Hessian log-Sobolev inequality}
\label{app:logsob}

In this section, we show that the bound in Theorem~\ref{thm:gradcon} may be sharpened under a Hessian log-Sobolev inequality (LSI), which we define presently.

\begin{definition}[Hessian log-Sobolev]
We say density $d\pis/dx \propto \exp(-f(x))$ satisfies a \emph{Hessian log-Sobolev inequality} if for all continuously differentiable $g: \R^d \rightarrow \R$, and for
\[\textup{Ent}_{\pis}\left[g\right] \defeq \left(\int g(x) \log\left(g(x)\right)d\pis(x) \right) - \left(\int g(x) d\pis(x) \right)\log\left(\int g(x)d\pis(x)\right), \]
we have
\[\textup{Ent}_{\pis}\left[g^2\right] \le 2\int_{\R^d} \inprod{\left(\nabla^2 f(x)\right)^{-1}\nabla g(x)}{\nabla g(x)} d\pis(x).\]
\end{definition}

In general, this is a much more restrictive condition than Theorem~\ref{thm:hesspoin}; some sufficient conditions are given in \cite{BobkovL00}. We now show an improved concentration result under a Hessian LSI; the proof follows Herbst's argument, a framework developed in \cite{Ledoux99}.

\begin{theorem}[Gradient norm concentration under LSI]
	\label{thm:gradconls}
	Suppose $f$ is $L$-smooth and strongly convex, and $d\pis(x)/dx \propto \exp(-f(x))$ satisfies a Hessian log-Sobolev inequality. Then for all $c > 0$,
	\[\Pr_{\pis}\left[\norm{\nabla f(x)} \ge \epis\left[\norm{\nabla f}\right] + c\sqrt{2L\log d} \right] \le d^{-c^2}.\]
\end{theorem}
\begin{proof}
Denote $G \defeq \norm{\nabla f}$, where we note $\nabla G = \frac{(\nabla^2 f) \nabla f}{\norm{\nabla f}}$. Let $H(\lambda) \defeq \epis\left[\exp(\lambda G)\right]$, such that $H'(\lambda) = \epis\left[G\exp(\lambda G)\right]$. Then, for $g^2 = \exp(\lambda G)$, 
\[H(\lambda) = \epis\left[g^2\right],\; \lambda H'(\lambda) = \epis\left[g^2\log g^2\right].\]
This in turn implies via the LSI that
\begin{equation}\label{eq:LSI}\lambda H'(\lambda) - H(\lambda) \log H(\lambda) = \epis\left[g^2 \log g^2\right] - \epis\left[g^2\right]\log\epis\left[g^2\right] \le 2 \epis\left[\norm{\nabla g}_{(\nabla^2 f)^{-1}}^2\right].\end{equation}
By smoothness and the definition of $g = \exp(\half \lambda G)$, we may bound the right hand side:
\begin{equation}\label{eq:gdef}\epis\left[\norm{\nabla g}_{(\nabla^2 f)^{-1}}^2\right] = \frac{\lambda^2}{4}\epis\left[\norm{\nabla G}_{(\nabla^2 f)^{-1}}^2\exp(\lambda G)\right] \le \frac{\lambda^2 L}{4}H(\lambda).\end{equation}
In the last inequality we used our calculation of $\nabla G$, and $\nabla^2 f \preceq LI_d$. Now, consider the function $K(\lambda) = \frac{1}{\lambda} \log H(\lambda)$. We handle the definition of $K(0)$ by a limiting argument (and $\log(1 + x) \approx x$):
\[K(0) = \lim_{\lambda \rightarrow 0} \frac{1}{\lambda} \log \epis\left[e^{\lambda G}\right] = \frac{H(\lambda) - H(0)}{\lambda} = H'(0) = \epis\left[G\right]. \]
We compute
\[K'(\lambda) = -\frac{1}{\lambda^2}\log H(\lambda) + \frac{H'(\lambda)}{\lambda H(\lambda)} = \frac{\lambda H'(\lambda) - H(\lambda)\log H(\lambda)}{\lambda^2 H(\lambda)}. \]
This, combined with \eqref{eq:LSI} and \eqref{eq:gdef} imply $K'(\lambda) \leq \frac{L}{2}$. Therefore, by integrating, we have
\[K(\lambda) \le \epis\left[G\right] + \frac{L\lambda}{2} \Rightarrow H(\lambda) = \exp(\lambda K(\lambda)) \le \exp\left(\lambda \epis\left[G\right] + \frac{L\lambda^2}{2}\right).\]
Finally, we have concentration:
\begin{align*}
\Pr_{\pis}\left[G \ge \epis\left[G\right] + r\right] &= \Pr_{\pis}\left[\exp(\lambda G) \ge \exp(\lambda \epis\left[G\right] + \lambda r)\right] \le \exp\left(-\lambda r + \frac{L\lambda^2}{2} \right),
\end{align*}
where the last statement is by Markov. Choosing $r = c\sqrt{2L\log d}$, $\lambda = r/L$ yields the conclusion.
\end{proof}
	

\section{Mixing time proofs}
\label{app:mixing}

We prove various claims from Section~\ref{sec:warmstart}; notation here is consistent with definitions in the body of the paper. All definitions will be with respect to some reversible random walk on $\R^d$ with transition distributions $\tran_x$ and stationary distribution $d\pis$. We use $d\pi_k$ to denote the density after $k$ steps.

\subsection{Blocking conductance framework}

In this section, we recall the \emph{blocking conductance} framework of \cite{KannanLM06}. This section is a restatement of their results for our purposes.

\subsubsection{Preliminaries}

Let $d\rho_k$ denote the ``average'' distribution density after $k$ steps, e.g. $d\rho_k(x) = \frac{1}{k}\sum_{0 \le i < k}d\pi_i(x)$. Define the flow between two sets $S, T \subseteq \R^d$ by
\[Q(S, T) \defeq \int_S \tran_x(T) d\pis(x),\; Q(S) \defeq Q(S, S^c).\]
For every $S \subseteq \R^d$ and $0 \le x \le \pis(S)$, define the conductance function by
\[\Psi(x, S^c) \defeq \min_{T \subseteq S, \pis(T) = x} Q(T, S^c). \]
In other words, $\Psi(x, S^c)$ is the smallest amount of flow from subsets of $S$ with stationary measure $x$ to $S^c$. It is clear that $\Psi(x, S^c)$ is monotone increasing in $x$ in the range $0 \le x \le \pis(S)$, as choosing a subset of larger measure can only increase flow. By convention, for $1 \ge x \ge \pis(S)$,
\[\Psi(x, S^c) \defeq \Psi(1 - x, S).\]
This definition clearly makes sense because $x \ge \pis(S) \Rightarrow 1 - x \le \pis(S^c)$. Next, let the \emph{spread} of $S \subseteq \R^d$ be defined as
\begin{equation}\label{eq:spreaddef}\psi(S) \defeq \int_0^{\pis(S)} \Psi(x, S^c) dx.\end{equation}
In other words, we can think of $\psi(S)$ as the worst-case flow between a subset of $S$ and $S^c$, where the measure of the subset is averaged uniformly over $[0, \pis(S)]$. The spread enjoys the following useful property, which allows us to think of the spread as a notion of conductance.

\begin{lemma}
\label{lem:spreadcond}
For any set $S \subseteq \R^d$,
\[\psi(S) \ge \frac{1}{4}Q(S)^2.\]
\end{lemma}
\begin{proof}
We claim first that for any $t \in [0, \pis(S)]$,
\begin{equation}\label{eq:monotonet} \psi(S) \ge (\pis(S) - t)\Psi(t, S^c). \end{equation}
To see this, we integrated only in the range $[t, \pis(S)]$, and used monotonicity of $\Psi(\cdot, S^c)$. Let $\gamma(S)$ denote the minimum measure of a subset $R$ of $S$, such that $Q(R, S^c) = \half Q(S, S^c)$. Note that this means any set $T$ with measure $\pis(S) - \gamma(S)$ has flow $Q(T, S^c)$ at least
\[Q(S) - \half Q(S) = \half Q(S).\]
In \eqref{eq:monotonet}, let $t = \pis(S) - \gamma(S)$, and let $T$ be the subset which admits the value $\Psi(t, S^c)$, i.e. such that $Q(T, S^c) = \Psi(t, S^c)$ and $\pis(T) = t$. In particular, this implies
\[\psi(S) \ge \gamma(S)Q(T, S^c) \ge \half\gamma(S)Q(S).\]
To show the conclusion, it suffices to show $\gamma(S) \ge Q(S)/2$. This is clear because if $\gamma(S) < Q(S)/2$, then any set of stationary measure $\gamma(S)$ could not absorb a flow of $Q(S)/2$ from the set $S^c$.
\end{proof}
The final definitions we will need are as follows. For an iteration $k$, let
\[g_k(x) \defeq k\frac{d\rho_k}{d\pis}(x) = \sum_{0 \le i < k} \frac{d\pi_i}{d\pis}(x).\]
A useful interpretation is that $\int_S g_k(x) d\pis(x)$ measures how many times the set $S$ was visited on expectation over the first $k$ iterations. Let $m_k: \R^d \rightarrow [0, 1]$ be a measure-preserving map such that $g_k(m_k^{-1}(\cdot))$ is an increasing function. In other words, $m_k$ orders the space $\R^d$ by their value $g_k$, such that for $0 \le s < t \le 1$, $g_k(m_k^{-1}(s)) \le g_k(m_k^{-1}(t))$. We define
\begin{equation}\label{eq:qdef}q_k(x, y) = Q\left(m_k^{-1}([0, \min(x, y)]), m_k^{-1}([\max(x, y), 1])\right).\end{equation}
In other words, for $x \le y$, $q_k$ takes the set of measure $x$ according to $d\rho_k/d\pis$ of least probability, and of measure $1 - y$ of most probability, and measures the flow between them. For notational simplicity and when clear from context, we identify $x \in [0, 1]$ with $m_k^{-1}(x)$, and similarly identify intervals. The following is then immediate:
\begin{equation}\label{eq:qderiv}\frac{d}{dx}q_k(x, y) = \begin{cases}
\tran_{x}([y, 1]) & x < y \\
-\tran_{x}([0, y]) & x \ge y
\end{cases}.\end{equation}

\subsubsection{Main claim}

Here, we recall the main result of the blocking conductance framework in terms of mixing times.

\begin{theorem}
	\label{thm:warmstart}
	Suppose the starting distribution $\pi_0$ is $\beta$-warm with respect to $\pis$. Let $h: [c, 1 - c] \rightarrow \R_{\ge 0}$ satisfy, for some $c \in (0, \half)$, and some $k$,
	\begin{equation}\label{eq:mixweightc} \int_c^{1-c} h(y) q_k(x, y) dy \ge 2x(1 - x),\; \forall x \in [c, 1 - c]. \end{equation}
	Then, 
	\[\norm{\rho_k - \pis}_{\textup{TV}} \le \beta c + \frac{1}{k}\int_{c}^{1 - c} h(x)dx.\]
\end{theorem}

We call a function $h$ which satisfies \eqref{eq:mixweightc} a $c$-mixweight function, and show how to construct such a function in Section~\ref{ssec:mixweight}; they will be inversely related to standard notions of conductance. We first require the following helper results.

\begin{lemma}
	\label{lem:diffpi}
	For any $t \in [0, 1]$,
	\[\int_0^1 q_k(x, t) dg_k(x) = \pi_k([0, t]) - \pi_0([0, t]) = \pi_0([t, 1]) - \pi_k([t, 1]). \]
\end{lemma}
\begin{proof}
	The equality $\pi_k([0, t]) - \pi_0([0, t]) = \pi_0([t, 1]) - \pi_k([t, 1])$ follows by definition, so we will simply show the first equality. Using \eqref{eq:qderiv} for $x \le t$ and integrating by parts,
	\[\int_0^t g_k(x)\tran_x([t, 1]) d\pis(x) = g_k(t)q(t, t) - \int_0^t q_k(x, t) dg_k(x).\]
	Similarly,
	\[\int_t^1 g_k(x) \tran_x([0, t]) d\pis(x)= g_k(t)q(t, t) + \int_t^1 q_k(x, t) dg_k(x). \]
	Therefore, to derive the conclusion of the lemma, it suffices to show that
	\[\int_t^1 g_k(x) \tran_x([0, t]) d\pis(x) - \int_0^t g_k(x)\tran_x([t, 1]) d\pis(x) = \pi_k([0, t]) - \pi_0([0, t]). \]
	By expanding the definition of $g_k$ and telescoping, it suffices to show for all $0 \le i \le k - 1$,
	\[\int_t^1 \tran_x([0, t]) d\pi_i(x) - \int_0^t \tran_x([t, 1]) d\pi_i(x) = \pi_{i + 1}([0, t]) - \pi_i([0, t]). \]
	This follows from
	\[\pi_{i + 1}([0, t]) - \pi_i([0, t]) = \left(\int_0^t \left(1 - \tran_x([t, 1])\right)d\pi_i(x) + \int_t^1 \tran_x([0, t])d\pi_i(x)\right) - \int_0^t d\pi_i(x).\]
\end{proof}

Next, let $t_0 \in [0, 1]$ be such that $g_k(t_0) = k$, where we note that $\E_{\pis}[g_k] = k$ is the expected value. 

\begin{lemma}
	\label{lem:tvk}
	\[\norm{\rho_k - \pis}_{\textup{TV}} = \frac{1}{k}\int_0^{t_0}t dg_k(t) = \frac{1}{k}\int_{t_0}^1 (1 - t)dg_k(t). \]
\end{lemma}
\begin{proof}
	By the definition of $t_0$, we have that for all $t \le t_0$, $\frac{d\rho_k}{d\pis}(x) \le 1$, and for $t \ge t_0$, $\frac{d\rho_k}{d\pis}(x) \ge 1$. Therefore, the total variation distance is attained by the set $[0, t_0]$, i.e.
	\begin{equation}\label{eq:tvattained}\norm{\rho_k - \pis}_{\textup{TV}} = \pis([0, t_0]) - \rho_k([0, t_0]) = \int_0^{t_0}\left(1 - \frac{g_k(x)}{k}\right) d\pis(x). \end{equation}
	Integrating by parts,
	\[\norm{\rho_k - \pis}_{\textup{TV}} = \left(1 - \frac{g_k(t_0)}{k}\right)t_0 + \frac{1}{k} \int_0^{t_0} tdg_k(t). \]
	The first summand vanishes by the definition of $t_0$, so we attain the first equality in the lemma statement. The second equality follows from the same calculations, using the set $[t_0, 1]$ which also attains the total variation distance, i.e. integrating by parts
	\[\int_{t_0}^1\left(\frac{g_k(x)}{k} - 1\right) d\pis(x) = \frac{1}{k}\int_{t_0}^1 (1 - t)dg_k(t).\]
\end{proof}

We also remark that for a $\beta$-warm start, it follows that every distribution $\pi_i$ for $i \ge 0$ is also $\beta$-warm, as the warmness $d\pi_{i + 1}/d\pis$ at a point is given by an average over the values $d\pi_i/d\pis$ of the prior iteration, and the conclusion follows by induction.

\begin{lemma}
	\label{lem:betac}
	\[\norm{\rho_k - \pis}_{\textup{TV}} \le\min\left(\beta c + \frac{1}{k}\int_c^{t_0}t dg_k(t),\; \beta c + \frac{1}{k}\int_{t_0}^{1 - c} (1 - t)dg_k(t)\right).\]
\end{lemma}
\begin{proof}
	Recall in Lemma~\ref{lem:tvk} we characterized
	\begin{equation*}\begin{aligned}\norm{\rho_k - \pis}_{\textup{TV}} &= \int_0^{t_0}\left(1 - \frac{g_k(x)}{k}\right) d\pis(x) \\
	&= \int_0^c\left(1 - \frac{g_k(x)}{k}\right) d\pis(x) + \int_c^{t_0}\left(1 - \frac{g_k(x)}{k}\right) d\pis(x). \end{aligned}\end{equation*}
	Note that the first integral is at most $c$. The second integral is, integrating by parts,
	\[\left(\frac{g_k(c)}{k} - 1\right) c + \frac{1}{k}\int_c^{t_0}tdg_k(t).\]
	The first summand is bounded by $(\beta - 1)c$ by our earlier argument about the warmness at every iteration being bounded by $\beta$. Finally, the second half of the lemma statement follows by considering the other characterization of the total variation based on $[t_0, 1]$, e.g. bounding
	\[\int_{t_0}^{1 - c}\left(\frac{g_k(x)}{k} - 1\right) dx + \int_{1 - c}^{1}\left(\frac{g_k(x)}{k} - 1\right) dx. \]
\end{proof}

\begin{proof}[Proof of Theorem~\ref{thm:warmstart}]
	First, if $t_0 \le c$, by \eqref{eq:tvattained}, the total variation distance is at most $c\beta$ as desired. A similar conclusion follows if $t_0 \ge 1 - c$ from the other characterization of total variation in Lemma~\ref{lem:betac}. We now consider when $t \in (c, 1 - c)$. By Lemma~\ref{lem:diffpi}, for all $y \in [c, 1 - c]$,
	\[1 \ge \pi_k([0, y]) \ge \int_c^{1 - c} q_k(x, y)dg_k(x). \]
	Multiplying by $h$ and integrating over the range $[c, 1 - c]$,
	\[\int_c^{1 - c} h(x) dx \ge \int_c^{1 - c} \left(\int_c^{1 - c} h(y) q_k(x, y)dy\right) dg_k(x) \ge \int_c^{1 - c} 2x(1 - x)dg_k(x). \]
	The second inequality recalled the requirement \eqref{eq:mixweightc}. By combining this with half of Lemma~\ref{lem:betac},
	\begin{align*}
	\norm{\rho_k - \pis}_{\textup{TV}} &\le \beta c + \frac{1}{k}\int_c^{t_0}x dg_k(x) \le \beta c + \frac{1}{2(1 - t_0)k} \int_c^{t_0}2x(1 - x)dg_k(x) \\
	&\le \beta c+ \frac{1}{2(1 - t_0)k} \int_c^{1 - c} h(x) dx.
	\end{align*}
	By using the other half of Lemma~\ref{lem:betac}, we may similarly conclude
	\[\norm{\rho_k - \pis}_{\textup{TV}}  \le \beta c+ \frac{1}{2t_0 k} \int_c^{1 - c} h(x) dx.\]
	The conclusion follows from combining these bounds, i.e. depending on if $t_0 \le \half$ or $t_0 \ge \half$.
\end{proof}

\subsection{Mixweight functions}
\label{ssec:mixweight}

In this section, we propose a function $h$ satisfying \eqref{eq:mixweightc}, and prove its correctness. First, we describe a useful sufficient condition.

\begin{lemma}
	\label{lem:suffmixwc}
	Suppose $h: [c, 1 - c] \rightarrow \R_{\ge 0}$ has $h(1 - y) = h(y)$, and
	\begin{equation}\label{eq:suffmixw} \int_c^{1 - c} h(y) \Psi(y, S^c) dy \ge 2\pis(S)(1 - \pis(S)),\; \forall S \subseteq \R^d: c \le \pis(S) \le \half. \end{equation}
	Then, $h$ satisfies \eqref{eq:mixweightc}.
\end{lemma}
\begin{proof}
	Note that for $c \le x \le \half$, choosing $S = m_k^{-1}([0, x])$ in \eqref{eq:suffmixw} yields
	\[2x(1 - x) \le \int_c^{1 - c} h(y) \Psi\left(y, m_k^{-1}([x, 1])\right) dy \le \int_c^{1-c} h(y) q_k(x, y)dy. \]
	The second inequality follows as for $x \ge y$, $q_k(x, y) \ge \Psi(y, m_k^{-1}([x, 1]))$ by definition, and for $y \ge x$, we use symmetry of $\Psi$, $q_k$ in their arguments. A similar argument holds for $x \ge \half$ by symmetry.
\end{proof}

We now define our $c$-mixweight function $h$:
\begin{equation}\label{eq:hdef}h(y) \defeq \begin{cases}
\max_{y \le \pis(S) \le \min\left\{2y, \half\right\}} \frac{4\pis(S)}{\psi(S)} & y \le \half\\
h(1 - y) & y \ge \half
\end{cases}.\end{equation}
In particular, note that for all $y \le \half$, by combining with Lemma~\ref{lem:spreadcond},
\begin{equation}\label{eq:squarebound} h(y) \le \max_{y \le \pis(S) \le 2y} \frac{16\pis(S)}{Q(S)^2}.\end{equation}
We will develop an upper bound on the ratio $\pis(S)/Q(S)^2$ for $\pis(S)$ which are ``not too small'' in the following section. We now prove correctness of the definition \eqref{eq:hdef}.

\begin{lemma}
The function $h$ defined in \eqref{eq:hdef} satisfies \eqref{eq:mixweightc}.
\end{lemma}
\begin{proof}
Recall that it suffices to show that $h$ satisfies \eqref{eq:suffmixw}. To this end, let $S$ be some set such that $\pis(S) = x \in [2c, \half]$. Then, recalling the definition of the spread $\psi(S)$ \eqref{eq:spreaddef},
\begin{align*}
2x &\le \frac{4\pis(S)}{2\psi(S)}\left(\int_0^x \Psi(y, S^c)dy\right) \\
&\le \frac{4\pis(S)}{\psi(S)}\int_{x/2}^x \Psi(y, S^c) dy\\
&\le \int_{x/2}^x h(y) \Psi(y, S^c) h(y)dy \\
&\le \int_{c}^{1 - c} h(y) \Psi(y, S^c) dy. 
\end{align*}
In the second line, we used the monotonicity of $\Psi(\cdot, S^c)$ in the first argument; in the third line, we used the definition of $h$ with the fact that $x \in [y, \min\{2y, \half\}]$ for all $y \in [x/2, x]$. To handle the case $x \in [c, 2c]$,
\begin{align*}
\int_{c}^{1 - c} h(y) \Psi(y, S^c) dy &\ge \int_{x}^{3x/2} h(y)\Psi(y, S^c) dy\\
&\ge \int_{x/2}^x h(y)\Psi(y, S^c)dy,
\end{align*}
where the second line is due to monotonicity, and the rest of the proof proceeds as before.
\end{proof}

\begin{proof}[Proof of Theorem~\ref{thm:blockmix}]
This follows from combining Theorem~\ref{thm:warmstart} with our particular choice of mixweight function given in \eqref{eq:hdef}, whose denominator we bound via \eqref{eq:squarebound}. Because $h$ is symmetric, it suffices to double the integration from $c$ to $\half$, and the bounds within the integral come from monotonicity of $\phi$.
\end{proof}

\subsection{Restricted conductance via total variation bounds}

For $S \subseteq \R^d$ and $x \in \R^d$, we define $d(S, x) \defeq \min_{y \in S}\norm{x - y}$; for $S_1, S_2 \subseteq \R^d$, $d(S_1, S_2) \defeq \min_{x \in S_2} d(S_1, x)$. The following isoperimetric inequality was given as Lemma 12 of \cite{ChenDWY19}.

\begin{lemma}[Logarithmic isoperimetric inequality]
\label{lem:logiso}
Let $\pis$ be any $\mu$-strongly logconcave function. For any partition $A_1, A_2, A_3$ of $\R^d$ with $\pis(A_1) \le \pis(A_2)$,
\[\pis(A_3) \ge \frac{d(A_1, A_2)\sqrt{\mu}}{2}\pis(A_1)\log^{\half}\left(1 + \frac{1}{\pis(A_1)}\right).\]
\end{lemma}

\begin{lemma}
\label{lem:conductance}
Let $\pis$ be any $\mu$-strongly logconcave function, and let $\delta\sqrt{\mu} < 1$ for some $\delta > 0$. Suppose for $\Omega \subset R^d$ with $\pis(\Omega) = 1 - s$, and all $x, y \in \Omega$ with $\norm{x - y} \le \delta$, 
\[\norm{\tran_x - \tran_y}_{\textup{TV}} \le 1 - \alpha.\]
Then, for all $s \le t \le \half$ and $S$ with $\pis(S) = t$,
\[\frac{\pis(S)}{Q(S)^2} \le \frac{16t}{\alpha^2}\left(\frac{\delta\sqrt{\mu}}{4}(t - s)\log^{\half}\left(1 + \frac{1}{t}\right) - s\right)^{-2}.\]
In particular, if 
\begin{equation}\label{eq:reallysmalls} s \le \min\left(\frac{t}{2}, \frac{\delta\sqrt{\mu}t}{16}\sqrt{\log(3)}\right),\end{equation}
we have the simplified bound
\[\frac{\pis(S)}{Q(S)^2} \le \frac{2^{16}}{\alpha^2\delta^2\mu t\log(1/t)}.\]
\end{lemma}
\begin{proof}
Let $S$ have $\pis(S) = t$. Define the following three sets:
\[A_1 \defeq \left\{x \in S \cap \Omega \mid \tran_x(S^c) < \frac{\alpha}{2}\right\},\; A_2 \defeq \left\{x \in S^c \cap \Omega \mid \tran_x(S) < \frac{\alpha}{2}\right\},\; A_3 \defeq (A_1 \cup A_2)^c.\]
Note that for any $x \in A_1$, $y \in A_2$, we have $\norm{\tran_x - \tran_y}_{\textup{TV}} > 1 - \alpha$, and therefore $\norm{x - y} > \delta$. Moreover, if $\pis(A_1) < \half \pis(S)$, 
\[Q(S) \ge \frac{\alpha}{4}(t - s). \]
Similarly, if $\pis(A_2) < \half \pis(S^c \cap \Omega)$:
\[Q(S) \ge \frac{\alpha}{4}(1 - t - s).\] 
These bounds are subsumed by the third case, where $\pis(A_1) \ge \half \pis(S)$, $\pis(A_2) \ge \half \pis(S^c \cap \Omega)$. By Lemma~\ref{lem:logiso}, since we argued $d(A_1, A_2) > \delta$,
\begin{align*}\pis(A_3) &\ge \frac{\delta\sqrt{\mu}}{2}\min(\pis(A_1), \pis(A_2))\log^{\half}\left(1 + \frac{1}{\min(\pis(A_1), \pis(A_2))}\right) \\
&\ge \frac{\delta\sqrt{\mu}}{4}\min\left(\pis(S \cap \Omega), \pis(S^c \cap \Omega)\right)\log^{\half}\left(1 + \frac{1}{\min\left(\pis(S \cap \Omega), \pis(S^c \cap \Omega)\right)}\right) \\
&\ge  \frac{\delta\sqrt{\mu}}{4}(t - s)\log^{\half}\left(1 + \frac{1}{t}\right). \end{align*}
This immediately implies
\[\pis(A_3 \cap \Omega) \ge \frac{\delta\sqrt{\mu}}{4}(t - s)\log^{\half}\left(1 + \frac{1}{t}\right) - s.\]
Finally, by the definition of stationary distribution,
\begin{align*}
Q(S) &= \half\left(\int_S \tran_x(S^c)d\pis(x) + \int_{S^c} \tran_x(S)d\pis(x)\right) \\
&\ge \half\int_{A_3 \cap \Omega} \frac{\alpha}{2}d\pis(x) = \frac{\alpha}{4}\pis(A_3 \cap \Omega) \\
&\ge \frac{\alpha}{4}\left(\frac{\delta\sqrt{\mu}}{4}(t - s)\log^{\half}\left(1 + \frac{1}{t}\right) - s\right).
\end{align*}
If \eqref{eq:reallysmalls} holds, we have the improved bound
\[Q(S) \ge \frac{\alpha\delta\sqrt{\mu}}{64}t\log^{\half}\left(\frac{1}{t}\right).\]
\end{proof}

\subsection{Exponential convergence with a warm start}
\label{app:reduction}

In this section, we give a simple reduction from a bound on the number of iterations it takes a Markov chain to attain constant total variation distance to the stationary distribution from a warm start, to the number of iterations it takes for the distance to decrease to $\eps$, with logarithmic dependence on $\eps$. Throughout, $\pis$ is the stationary distribution of a Markov chain with transitions $\tran$, and we let $\tran^k \pi$ be the result of running $k$ steps of the chain from starting distribution $\pi$. For specified $\pi_0$, we denote $\pi_k \defeq \tran^k \pi_0$. Sppose we have a bound of the following type. 

\begin{assumption}
	\label{assume:warm}
	$\exists \tmix$ such that for every $\pi$ which is $\beta/\eps$-warm with respect to $\pis$,
	\[\norm{\tran^{\tmix}\pi - \pis}_{\textup{TV}} \le \frac{1}{2e}.\]
\end{assumption}

We first recall some basic facts about the optimal \emph{coupling} between two distributions $\pi$, $\rho$, which informally is the joint distribution $\mu$ with the prescribed marginals $\pi$ and $\rho$ which maximizes the probability that for $(x, y) \sim \mu$, $x = y$. For a reference, see \cite{LevinPW09}.

\begin{fact}
\label{fact:coupling}
Let $\mu$ be the optimal coupling between distributions $\pi$ and $\rho$. The following hold.
\begin{enumerate}
	\item $\Pr_{(x, y) \sim \mu}[x \neq y] = \norm{\pi - \rho}_{\textup{TV}}$.
	\item Consider the marginal distribution of $(x, y) \sim \mu$ in the first variable, conditioned on $x \neq y$. It has a density proportional to $d\pi(x) - \min(d\pi(x), d\rho(x))$.
\end{enumerate}
\end{fact}

The following result is well-known.

\begin{lemma}
	\label{lem:tvmono}
	For any distribution $\pi$,
	\[\norm{\tran\pi - \pis}_{\textup{TV}} \le \norm{\pi - \pis}_{\textup{TV}}.\]
\end{lemma}
\begin{proof}
	Consider the optimal coupling $\mu$ between $\pi$ and $\pis$, and note that 
	\[\Pr_{(x, y) \sim \mu}[x \neq y] = \norm{\pi - \pis}_{\textup{TV}}.\]
	It follows that the optimal coupling between $\mu'$ between $\tran\pi$ and $\pis$ has \[\Pr_{(x, y) \sim \mu'}[x \neq y] \le \Pr_{(x, y) = \mu}[x \neq y],\]
	since $\tran\pis = \pis$, and with probability $\Pr_{(x, y) \sim \mu}[x = y]$ the coupling $\mu'$ can keep $x$ and $y$ coupled.
\end{proof}

\begin{lemma}
	Under Assumption~\ref{assume:warm}, letting $\pi_0$ be a $\beta$-warm start, and $k \ge \tmix\log(\eps^{-1})$,
	\[\norm{\tran^k\pi_0 - \pis}_{\textup{TV}} \le \eps.\]
\end{lemma}
\begin{proof}
	Assume for the sake of contradiction that $\norm{\pi_k - \pis}_{\textup{TV}} > \eps$; note that by Lemma~\ref{lem:tvmono}, this implies that $\norm{\pi_i - \pis}_{\textup{TV}} > \eps$ for all $i \le k$. For any $i$, we denote $\mu_i$ to be the best coupling between $\pi_i$ and $\pis$. Note that for any $i$, we can compute the marginal conditional distribution of the uncoupled set of $\pi_i$, under the coupling $\mu_i$, by Fact~\ref{fact:coupling}:
	\[\frac{d\tilde{\pi}_i}{d\pis}(x) \defeq \frac{\frac{d\pi_i}{d\pis}(x) - \min\left(\frac{d\pi_i}{d\pis}(x), 1\right)}{\int \left(\frac{d\pi_i}{d\pis}(x) - \min\left(\frac{d\pi_i}{d\pis}(x), 1\right)\right)d\pis(x)} \le \frac{\frac{d\pi_i}{d\pis}(x)}{\norm{\pi_i - \pis}_{\textup{TV}}} \le \frac{\beta}{\eps}. \]
	Here, we used the observation that if $\pi_0$ is $\beta$-warm, then so are all $\pi_i$ for $i \ge 0$. Similarly, the conditional distribution of the uncoupled set of $\pis$ under $\mu_i$ satisfies
	\[\frac{d\tilde{\pi}^*_i}{d\pis}(x) \defeq \frac{1 - \min\left(\frac{d\pi_i}{d\pis}(x), 1\right)}{\int \left(1 - \min\left(\frac{d\pi_i}{d\pis}(x), 1\right)\right)d\pis(x)} \le \frac{1}{\eps}. \]
	This implies the conditional distributions $\tpi_i$ and $\tpi_i^*$ are both $\beta/\eps$-warm with respect to $\pis$ for any $i \le k$. After $\tmix$ iterations, the total variation distance between $\tpi_i$ and $\tpi_i^*$ is bounded by $1/e$ by Assumption~\ref{assume:warm} and the triangle inequality. Repeating this argument $\log(\eps^{-1})$ times implies that the measure of the uncoupled set decreases by at least a $1/e$ factor between iterations $i$ and $i + \tmix$, while $i \le k$, so that the uncoupled set has measure at most $\eps$ by iteration $k$. Recalling that the measure of the uncoupled set is precisely the distance $\norm{\pi_k - \pis}_{\textup{TV}}$ results in a contradiction.
\end{proof}

	\section{Total variation bounds}
\label{app:totalvariation}

In this section, we prove the following lemma, which is the key step
in lower bounding the conductance of one step of our algorithm. 
\begin{lemma}
	\label{lem:etatvbound}
	For $\eta^2\leq\frac{1}{20Ld\log\frac{\kappa}{\eps}}$, the Markov chain defined in Algorithm~\ref{alg:mhmc} satisfies 
	\begin{equation}\label{eq:pvsp}
	\sup_{\left\Vert x-y\right\Vert \leq\eta}\tvd{\prop_x}{\prop_{y}}\leq\frac{5}{8} 
	\end{equation}
	and, for $\Omega$ defined in \eqref{eq:omegadef},
	\begin{equation}\label{eq:pvst}
	\sup_{x\in\Omega}\tvd{\prop_x}{\tran_x}\leq\frac{1}{8}.
	\end{equation}
\end{lemma}

\begin{proof}
	We first show \eqref{eq:pvsp}. From any point $x \in \R^d$, let $\tx$ be the proposed point according to Algorithm~\ref{alg:mhmc}; we recall that the update is given by, for $v \sim \Nor(0, I_d)$,
	\[
	\tilde{x}\gets x+\eta v-\frac{\eta^{2}}{2}\nabla f(x) \Rightarrow \tilde{x}\sim\mathcal{N}\left(x-\frac{\eta^{2}}{2}\nabla f(x),\eta^{2}I_{d}\right).
	\]
	Therefore, recalling that the KL divergence $d_{KL}$ between two Gaussians with covariance $\sigma^2 I_d$ and means $\mu_x$, $\mu_y$ is $\norm{\mu_x - \mu_{y}}^2/2\sigma^2$, Pinsker's inequality implies 
	\begin{align*}
	\tvd{\prop_x}{\prop_{y}}& \leq\sqrt{\frac{1}{2}d_{KL}\left(\mathcal{P}_{x},\mathcal{P}_{y}\right)}\\
	& \leq\frac{\norm{\left(x - \frac{\eta^2}{2}\nabla f(x)\right) - \left(y - \frac{\eta^2}{2}\nabla f(y)\right)}}{2\eta}\\
	& \leq\frac{\left(1 + \frac{L\eta^2}{2}\right)\norm{x - y}}{2\eta} \le \frac{5}{8},
	\end{align*}
	for $\norm{x - y} \le \eta$ and $\eta^2 \le (2L)^{-1}$. The third line used the triangle inequality and $\nabla f$ is $L$-Lipschitz.
	
	Next, we show \eqref{eq:pvst}. From a point $x$, and for any proposed transition to $\tx \neq x$, the proposal $\prop_x$ places at least as much mass on $\tx$ as $\tran_x$, because the rejection probability is nonnegative; consequently, the set $A$ maximizing $\tran_x(A) - \prop_x(A)$ is the singleton $A = \{x\}$, and the total variation distance is simply the probability $\tran_x = x$, or
	\begin{align*}\tvd{\prop_x}{\tran_x} &= 1 - \E_{v\sim\Nor(0, I_d)}\left[\min\left\{1, \exp\left(\ham(x, v) - \ham(\tx, \tv)\right)\right\}\right] \\
	&\le 1 - \E_{v\sim\Nor(0, I_d)}\left[\exp\left(\ham(x, v) - \ham(\tx, \tv)\right)\right].\end{align*}
	Therefore, to show the desired $\tvd{\prop_x}{\tran_x} \le 1/8$, it suffices to show that 
	\[\E_{v\sim\Nor(0, I_d)}\left[\exp\left(\ham(x, v) - \ham(\tx, \tv)\right)\right] \ge \frac{7}{8}.\]
	By the calculation 
	\[\frac{15}{16}\cdot\exp\left(-\frac{1}{16}\right) \ge \frac{7}{8},\]
	it suffices to show that with probability $15/16$ over the randomness of $v$, $\ham(x, v) - \ham(\tx, \tv) \ge -1/16$. First, by a standard tail bound on the chi-squared distribution (Lemma 1 of \cite{LaurentM00}), we have
	\[
	\Pr\left[\left\Vert v\right\Vert ^{2}\geq d+ 2\sqrt{3d} + 6\right]\leq \exp(-3) \le \frac{1}{16}.
	\]
	Thus, assuming $d$ is at least a sufficiently large constant, with probability at least $1/16$ over the randomness of $v$, we have $\norm{v} \le \sqrt{2d}$. Finally, the conclusion follows from the claim
	\[\ham(\tx, \tv) - \ham(x, v) \le \frac{1}{16}, \; \forall x \in \Omega,\; \norm{v} \le \sqrt{2d},\]
	which we now show. Recalling $\tilde{v}=v-\frac{\eta}{2}\left(\nabla f(\tilde{x})+\nabla f(x)\right)$ and $\tilde{x}=x+\eta v-\frac{\eta^{2}}{2}\nabla f(x)$,
	\begin{align*}
	 \mathcal{H}\left(\tilde{x},\tilde{v}\right)-\mathcal{H}\left(x,v\right) &= -\frac{1}{2}\left\Vert v\right\Vert ^{2}+\frac{1}{2}\left\Vert \tilde{v}\right\Vert ^{2}-f(x)+f(\tilde{x})\\
	&\leq -\frac{1}{2}\norm{v}^2 + \half \norm{\tv}^2+\frac{1}{2}\inprod{\nabla f(\tx) + \nabla f(x)}{\tx - x}+\frac{L}{4}\left\Vert \tilde{x}-x\right\Vert ^{2}\\
	&= \half\norm{v - \frac{\eta}{2}(\nabla f(x) + \nabla f(\tx))}^2 - \half\norm{v}^2 \\
	&+\frac{1}{2}\inprod{\nabla f(\tx) + \nabla f(x)}{\eta v - \frac{\eta^2}{2}\nabla f(x)}+\frac{L}{4}\left\Vert \tilde{x}-x\right\Vert ^{2}\\
	&= -\frac{\eta}{2}\inprod{\nabla f(x) + \nabla f(\tx)}{v} + \frac{\eta^2}{8}\norm{\nabla f(x) + \nabla f(\tx)}^2\\
	&+ \frac{1}{2}\inprod{\nabla f(\tx) + \nabla f(x)}{\eta v - \frac{\eta^2}{2}\nabla f(x)}+\frac{L}{4}\left\Vert \tilde{x}-x\right\Vert ^{2}\\
	&= \frac{\eta^2}{8}\inprod{\nabla f(x) + \nabla f(\tx)}{\nabla f(\tx) - \nabla f(x)}+\frac{L}{4}\left\Vert \tilde{x}-x\right\Vert ^{2}\\
	&\le \frac{\eta^2 L}{8}\norm{x - \tx}\norm{\nabla f(x) + \nabla f(\tx)} + \frac{L}{4}\norm{x - \tx}^2.
	\end{align*}
	The second inequality followed from
	\[f(\tx) - f(x) \le \min\left(\inprod{\nabla f(\tx)}{\tx - x},\; \inprod{\nabla f(x)}{\tx - x} + \frac{L}{2}\norm{\tx - x}^2\right),\] 
	due to convexity and smoothness; the last inequality followed from smoothness and Cauchy-Schwarz, and every other line was by expanding the definitions. We now bound these two terms. First, since  smoothness implies $\norm{\nabla f(x) + \nabla f(\tx)} \le 2\norm{\nabla f(x)} + L\norm{x - \tx}$, 
	\begin{align*}\frac{\eta^2 L}{8}\norm{x - \tx}\norm{\nabla f(x) + \nabla f(\tx)} &\le \frac{\eta^2 L^2}{8}\norm{x - \tx}^2 + \frac{\eta^2 L}{4}\norm{\nabla f(x)}\norm{x - \tx} \\ 
	&\le \frac{ L}{4}\norm{x - \tx}^2 + \frac{\eta^2 L}{4}\norm{\nabla f(x)}\norm{x - \tx}\end{align*}
	Here we used our choice of $\eta$. Next, since $\tx - x = \eta v -\frac{\eta^2}{2}\nabla f(x)$, using the above bounds,
	\begin{align*}\ham(\tx, \tv) - \ham(x, v)&\le \frac{\eta^2 L}{8}\norm{x - \tx}\norm{\nabla f(x) + \nabla f(\tx)} + \frac{L}{4}\norm{x - \tx}^2\\
	&\le \frac{L}{2}\norm{x - \tx}^2 + \frac{\eta^2 L}{4}\norm{\nabla f(x)}\norm{x - \tx}\\
	&\le L\eta^2\norm{v}^2 + \frac{L\eta^4}{4}\norm{\nabla f(x)}^2 + \frac{L\eta^3}{4}\norm{\nabla f(x)}\norm{v} + \frac{L\eta^4}{8}\norm{\nabla f(x)}^2 \\
	&\le \frac{9L\eta^2}{8}\norm{v}^2 + \frac{L\eta^4}{2}\norm{\nabla f(x)}^2 \le \frac{1}{16}.
	\end{align*}
	We recalled $\norm{v}^2 \le 2d$, $\norm{\nabla f(x)}^2 \le 25Ld^2\log^2\frac{\kappa}{\eps}$, and the choice of $\eta$.
\end{proof}

Finally, we note that Lemma~\ref{lem:tvst} immediately follows via the triangle inequality.
	\section{Deferred proofs}
\label{app:deferred}

\begin{lemma}
For any $C < 1$,
\[\prod_{k = 0}^\infty \left(\frac{1}{1 - \frac{C}{4^k}} \right)^{2^k} \le \frac{1 + \sqrt{C}}{1 - \sqrt{C}}.\]
\end{lemma}
\begin{proof}
Define
\[V(C) \defeq \prod_{k = 1}^\infty \left(\frac{1}{1 - \frac{C}{4^k}} \right)^{2^k} \le \frac{1 + \sqrt{C}}{1 - \sqrt{C}},\]
so we wish to bound $V(C)/(1 - C)$. It suffices to show $V(C) \le (1 + \sqrt{C})^2$. Note that
\begin{align*}\log V(C) &= \sum_{k = 1}^\infty 2^k\log\left(\frac{1}{1 - \frac{C}{4^k}}\right) = \sum_{k = 1}^\infty 2^k \sum_{j = 1}^\infty \frac{1}{j}\left(\frac{C}{4^k}\right)^j = \sum_{j = 1}^\infty \frac{C^j}{j(2^{2j - 1} - 1)}.
\end{align*}
Thus, $\log V$ is a convex function in $C$. Note that $\log V(0) = 0$ and 
\[\log V(1) \le 1 + \sum_{j = 2}^\infty \frac{1}{4^{j - 1}j} = 1 + 4\left(-\log\left(\frac{3}{4}\right) - \frac{1}{4}\right) \le \log 4.\]
This implies $\log V(C) \le C\log 4$, and the conclusion follows from $4^C \le (1 + \sqrt{C})^2$ for $C \in [0, 1]$.
\end{proof}

	\end{appendix}
	
\end{document}